\documentclass{article}

\PassOptionsToPackage{numbers, sort&compress}{natbib}

\usepackage[final]{neurips_2021}

\usepackage[utf8]{inputenc} %
\usepackage[T1]{fontenc}    %
\usepackage{hyperref}       %
\usepackage{url}            %
\usepackage{booktabs}       %
\usepackage{amsfonts}       %
\usepackage{nicefrac}       %
\usepackage{microtype}      %
\usepackage{xcolor}         %

\title{Structure learning in polynomial time: Greedy algorithms, Bregman information, and exponential families}

\author{%
	Goutham Rajendran\\
	University of Chicago\\
	\texttt{goutham@uchicago.edu}
	\And
	Bohdan Kivva\\
	University of Chicago\\
	\texttt{bkivva@uchicago.edu}
	\And
	Ming Gao\\
	University of Chicago\\
	\texttt{minggao@uchicago.edu}
	\And
	Bryon Aragam\\
	University of Chicago\\
	\texttt{bryon@chicagobooth.edu}
}

\usepackage{cmap}
\usepackage{textcomp}
\usepackage{comment}

\usepackage{xspace}
\usepackage{enumerate}
\usepackage{amssymb}
\usepackage{amsmath}
\usepackage{mathtools}

\usepackage{pgf, tikz}
\usetikzlibrary{arrows, automata}

\usepackage{bm}
\usepackage{bbm}
\usepackage{amsthm}
\usepackage[capitalise]{cleveref}

\usepackage{subcaption}

\allowdisplaybreaks

\newcommand{\EE}{\mathbb{E}}

\newcommand{\RR}{\mathbb{R}}

\newcommand{\calD}{\mathcal{D}}

\newcommand{\eps}{\epsilon}

\newcommand{\pr}{\mathbb{P}}

\newcommand{\R}{\mathbb{R}}

\newcommand{\E}{\mathbb{E}}

\renewcommand{\|}{\,|\,}            %
\newcommand{\given}{\,|\,}  %

\newcommand{\ip}[1]{\langle#1\rangle}

\DeclareMathOperator{\cov}{cov}
\DeclareMathOperator{\var}{var}

\DeclareMathOperator*{\argmin}{arg\,min}

\newcommand{\gr}{W} %
\DeclareMathOperator{\pa}{pa} %

\DeclareMathOperator{\prob}{\mathbb{P}}

\newtheorem{theorem}{Theorem}
\newtheorem{corollary}[theorem]{Corollary}
\newtheorem{example}[theorem]{Example}
\newtheorem{lemma}[theorem]{Lemma}

\newtheorem{definition}[theorem]{Definition}
\newtheorem{fact}[theorem]{Fact}
\newtheorem{remk}[theorem]{Remark}

\newtheorem{proposition}[theorem]{Proposition}
\newtheorem{assumption}[theorem]{Assumption}

\numberwithin{theorem}{section}

\newcommand{\dags}{\mathsf{DAG}}
\newcommand{\score}{S}
\newcommand{\LS}{\textup{LS}}
\newcommand{\ls}{\score_{\LS}}
\newcommand{\bs}{\score_{\phi}}
\DeclareMathOperator{\ERF}{ERF}
\DeclareMathOperator{\poly}{poly}

\usepackage[linesnumbered, ruled,vlined]{algorithm2e}

\begin{document}

	\maketitle

	\begin{abstract}
		Greedy algorithms have long been a workhorse for learning graphical models, and more broadly for learning statistical models with sparse structure.
		In the context of learning directed acyclic graphs, greedy algorithms are popular despite their worst-case exponential runtime. In practice, however, they are very efficient. We provide new insight into this phenomenon by studying a general greedy score-based algorithm for learning DAGs. Unlike edge-greedy algorithms such as the popular GES and hill-climbing algorithms, our approach is vertex-greedy and requires at most a polynomial number of score evaluations. We then show how recent polynomial-time algorithms for learning DAG models are a special case of this algorithm, thereby illustrating how these order-based algorithms can be rigorously interpreted as score-based algorithms. This observation suggests new score functions and optimality conditions based on the duality between Bregman divergences and exponential families, which we explore in detail. Explicit sample and computational complexity bounds are derived. Finally,
		we provide extensive experiments suggesting that this algorithm indeed optimizes the score in a variety of settings.
	\end{abstract}

	\section{Introduction}
	\label{sec:intro}

	Learning the structure of a graphical model from data is a notoriously difficult combinatorial problem with numerous applications in machine learning, artificial intelligence, and causal inference as well as scientific disciplines such as genetics, medicine, and physics.
	Owing to its combinatorial structure, greedy algorithms have proved popular and efficient in practice. For undirected graphical models (e.g. Ising, Gaussian) in particular, strong statistical and computational guarantees exist for a variety of greedy algorithms \cite{jalali2011learning,johnson2012high}. These algorithms are based on the now well-known forward-backward greedy algorithm \cite{liu2014forward,zhang2008adaptive}, which has been applied to a range of problems beyond graphical models including regression \cite{zhang2008adaptive}, multi-task learning \cite{tian2016forward}, and atomic norm regularization \cite{rao2015forward}.

	Historically, the use of the basic forward-backward greedy scheme for learning directed acyclic graphical (DAG) models predates some of this work, dating back to the classical greedy equivalence search \citep[GES,][]{chickering2003} algorithm. Since its introduction, GES has become a gold-standard for learning DAGs, and is known to be asymptotically consistent under certain assumptions such as faithfulness and score consistency \citep{chickering2003,nandy2018}. Both of these assumptions are known to hold for certain parametric families \cite{geiger2001}, however, extending GES to distribution-free settings has proven difficult. Furthermore, although GES is in practice extremely efficient and has been scaled up to large problem sizes \citep{ramsey2016}, it lacks polynomial-time guarantees.
	An important problem in this direction is the development of provably polynomial-time, consistent algorithms for DAG learning in general settings.

	In this paper, we revisit greedy algorithms for learning DAGs with an eye towards these issues. We propose a greedy algorithm for this problem---distinct from GES---and study its computational and statistical properties. In particular, it requires at most a polynomial number of score evaluations and provably recovers the correct DAG for properly chosen score functions. Furthermore, we illustrate its intimate relationship with existing order-based algorithms, providing a link between these existing approaches and classical score-based approaches. Along the way, we will see how the analysis itself suggests a family of score functions based on the Bregman information \cite{banerjee2005clustering}, which are well-defined without specific distributional assumptions.

	\paragraph{Contributions}
	At a high-level, our goal is to understand what kind of finite-sample and complexity guarantees can be provided for greedy score-based algorithms in general settings. In doing so, we aim to provide deeper insight into the relationships between existing algorithms. Our main contributions can thus be outlined as follows:
	\begin{itemize}
		\item A generic greedy forward-backward scheme for optimizing score functions defined over DAGs. Unlike existing \emph{edge-}greedy algorithms that greedily add or remove edges, our algorithm is \emph{vertex-}greedy, i.e. it greedily adds vertices in a topological sort.
		\item We show how several existing order-based algorithms from the literature are special cases of this algorithm, for properly defined score functions. Thus, we bring these approaches back under the umbrella of score-based algorithms.
		\item We introduce a new family of score functions derived from the Bregman information, and analyze the sample and computational complexity of our greedy algorithm for this family of scores.
		\item We explore the optimization landscape of the resulting score functions, and provide evidence that not only does our algorithm provably recover the true DAG, it does so by globally optimizing a score function.
	\end{itemize}
	The last claim is intriguing: It suggests that it is possible to globally optimize certain Bayesian network scores in polynomial-time. In other words, despite the well-known fact that global optimization of Bayesian networks scoring functions is NP-hard \citep{chickering2004,chickering1996}, there may be natural assumptions under which these hardness results can be circumvented. This is precisely the case, for example, for undirected graphs: In general, learning Markov random fields is NP-hard \citep{srebro2003maximum}, %
	but special cases such as Gaussian graphical models \cite{banerjee2008,misra2020information} and Ising models \cite{vuffray2016interaction,lokhov2018optimal,bresler2015efficiently} can be learned efficiently. Nonetheless, we emphasize that these results on global optimization of the score are merely empirical, and a proof of this fact beyond the linear case remains out of reach.

	\paragraph{Previous work}

	The literature on BNSL is vast, so we focus this review on related work involving score-based and greedy algorithms. For a broad overview of BNSL algorithms, see the recent survey \citep{glymour2019review} or the textbooks \citep{spirtes2000,peters2017elements}.
	The current work is closely related to and inspired by generic greedy algorithms such as \citep{jalali2011learning,johnson2012high,liu2014forward,rao2015forward,tian2016forward,zhang2008adaptive}.
	Existing greedy algorithms for score-based learning include
	GES \citep{chickering2003},
	hill climbing \citep{chickering1995learning,teyssier2012}, and
	A* search \citep{yuan2013}.
	In contrast to these greedy algorithms are global algorithms that are guaranteed to find a global optimum such as integer programming \citep{cussens2012,cussens2017} and dynamic programming \citep{ott2004,singh2005,silander2012}.
	Another family of order-based algorithms dating back to \cite{teyssier2012} centers around the idea of \emph{order search}---i.e. first searching for a topological sort---from which the DAG structure is easily deduced; see also \citep{geer2013,aragam2019globally,aragam2015highdimdag,raskutti2018,buhlmann2014}.
	Recently, a series of order-based algorithms have led to significant breakthroughs, most notable of which are finite-sample and strong polynomial-time guarantees \citep{ghoshal2017ident,chen2018causal,ghoshal2017sem,park2017,gao2020polynomial}. It will turn out that many of these algorithms are special cases of the greedy algorithm we propose; we revisit this interesting topic in Section~\ref{sec:eqvar}.

	\section{Background}
	\label{sec:bg}

	Let $X=(X_{1},\ldots,X_{d})$ be a random vector with distribution $\calD$. The goal of structure learning is to find a DAG $\gr=(V,E)$, also called a Bayesian network (BN), for the joint distribution $\calD$.
	Traditionally, there have been two dominant approaches to Bayesian network structure learning (BNSL): Constraint-based and score-based. In constraint-based algorithms such as the PC \citep{spirtes1991} and MMPC \citep{tsamardinos2006} algorithms, tests of conditional independence are used to identify the structure of a DAG via exploitation of $d$-separation in DAG models. Score-based algorithms such as GES \citep{chickering2003} define an objective function over DAGs such as the likelihood or a Bayesian posterior, and seek to optimize this score.

	To formalize this, denote the space of DAGs on $d$ nodes by $\dags$ and let $S:\dags\to\R$ be a score function. Intuitively, $S$ assigns to each DAG $\gr$ a ``score'' $S(\gr)$ that evaluates the fit between $\gr$ and $\calD$. In the sequel, we assume without loss of generality that the goal is to minimize the score:
	\begin{align}
		\label{eq:scorebased}
		\min_{\gr\in\dags} S(\gr).
	\end{align}
	Although this is an NP-hard combinatorial optimization problem, we can ask whether or not it is possible to design score functions $S$ which can be optimized efficiently, and whose minimizers are close to $\gr$. In order for this problem to be well-posed, there must be a \emph{unique} $\gr$ that we seek; namely, $\gr$ must be identifiable from $\calD$. The problem of identifiability will be taken up further in \Cref{sec: theoretical_guarantees}, where it will be connected to the choice of score function. For now, our primary interest is solving the problem \eqref{eq:scorebased}.

	Regarding score-based learning, we highlight a subtle point: Recovering the true DAG is not necessarily the same as minimizing the score function, for instance, see Example $1$ in \cite{loh2014causal}. Score-based algorithms in general attempt to learn the true model by way of minimizing the score but it’s possible that the graph which minimizes the score could be different from the true model. In other words, the score may not always be properly calibrated to the model. This is a well-studied problem, see e.g. \cite{geiger2001,geiger2002,heckerman1995}, and it is a fascinating and important open problem to better understand under what assumptions a score minimizer is also the true DAG in nonparametric settings.

	\paragraph{Exact algorithms}
	Solving problem \eqref{eq:scorebased} exactly (``exact'' meaning a genuine \emph{global} minimizer of \eqref{eq:scorebased} is returned) is known to be NP-hard \citep{chickering2004,chickering1996}.
	Some of the earliest exact methods for score-based learning relied on the following basic idea \citep{singh2005,ott2003,silander2012,perrier2008}: Use dynamic programming to search for optimal sinks in $\gr$, remove these sinks, and recursively find optimal sinks in the resulting subgraph. In doing so, a topological sort of $\gr$ can be learned, and from this sort, the optimal DAG can be easily learned. In other words, once the topological sort is known, finding the corresponding DAG is relatively easy. In the sequel, we refer to the problem of finding the topological sort of $\gr$ as \emph{order search}. Unfortunately, searching for optimal sinks involves computing $d 2^{d-1}$ local scores, which is both time and memory intensive.

	\paragraph{Poly-time algorithms}
	Recently, a new family of algorithms based on applying the idea of order search has led to significant breakthroughs in our understanding of this problem \citep{ghoshal2017ident,chen2018causal,ghoshal2017sem,park2017,gao2020polynomial}. Most notably, unlike the exact algorithms described above, these algorithms run in polynomial-time. The key distinction between these algorithms and exact algorithms is the clever exploitation of specific distributional (e.g. moments) or structural properties (e.g. linearity) of $\calD$, and as a result do not optimize a specific score function. In contrast, exact algorithms apply to \emph{any} score $S$, and do not require any distributional assumptions.

	\paragraph{Motivation}
	It is tempting to want to draw connections between exact algorithms and poly-time algorithms: After all, they both rely on the same fundamental principle of order search. In this paper, we explore this connection from the perspective of \emph{greedy optimization}. In particular, we will show how existing polynomial-time algorithms are special cases of a generic greedy forward-backward search algorithm for solving \eqref{eq:scorebased} under specific choices of $S$, and show how this leads to new insights for this problem. We \emph{do not} prove that this algorithm exactly solves \eqref{eq:scorebased} (save for the exceptional case of linear models; see Corollary~\ref{cor: gfbs_global}), however, we provide empirical evidence to support this idea on a variety of linear and nonlinear models in Section~\ref{sec: expt}. Since the score-based learning problem is NP-hard, this is of course not possible without additional assumptions.

	\paragraph{Notation}
	Let $n$ be the number of samples we observe. Each sample is a vector of the form $X = (X_1, \ldots, X_d)$ on $d$ variables. In this paper, $W$ is used for DAGs and the vertex set is $[d] = \{1, 2, \ldots, d\}$. Naturally, we match vertex $i$ to the variable $X_i$. We denote the set of parents of a vertex $i$ with $\pa_W(i)$, dropping the subscript when it's clear from context. We will also abuse notation and use $W$ for the adjacency matrix of the graph $W$ as well. Let $W^{(i)}$ denote the $i$th column of $W$, whose nonzero entries are precisely at the set of parents of vertex $i$. Let $W[j, k]$ denote the $(j, k)$th entry of the matrix.

	\section{The GFBS algorithm}
	\label{sec: main_algorithm}

	In this section, we will describe the greedy algorithm in a general framework.
	In subsequent sections, we will specialize to particular models or scores, as necessary.
	Throughout, we let $S$ be an arbitrary decomposable score. That is, $S(W) = \sum_{i \le d} S_i(W^{(i)})$ for functions $S_i$, an example of which would be the least-squares loss. All the score functions we study in the sequel will have this property.

	For a set of vertices $T$, let $e_T$ denote the indicator vector of $T$. For an edge $e$ of $W$, denote by $W^{-e}$ the matrix $W$ with the entry corresponding to $e$ zeroed out. For any set of vertices $J$ and vertex $i \not\in J$, denote by $W[J \to i]$ the matrix $W$ where the $i$th column $W^{(i)}$ is replaced by the indicator vector of $J$. That is,
	\[W^{-e}[j, k] = \begin{dcases} W[j, k] & \text{if $(j, k) \neq e$},\\ 0 & \text{otherwise}, \end{dcases} \qquad W[J \to i][j, k] = \begin{dcases} W[j, k] & \text{if $k \neq i$},\\ 1 & \text{if $k = i$ and $j \in J$}\\ 0 & \text{if $k = i$ and $j \not\in J$} \end{dcases}\]

	\begin{algorithm}[t]
		\DontPrintSemicolon
		\KwIn{Dataset $X$, tolerance parameter $\gamma \ge 0$}
		\KwOut{DAG $W$}
		$W = \emptyset$ \tcp{$n$-vertex graph with no edges}
		$T = []$\tcp{The ordering}
		\tcp{Forward phase}
		\For{$iter = 1$ to $d$}{
			$i = \argmin_{i \not\in T} S_i(e_T)$\tcp{Minimize jump in score}
			$W = W[T \to i]$\\
			$T.append(i)$\;
		}
		\tcp{Backward phase}
		\For{edge $e$ in $W$}{
			\If{$S(W^{-e}) - S(W) \le \gamma$}{
				$W = W^{-e}$\tcp{Delete the edge $e$}
			}
		}
		\Return{$W$}\tcp{Guaranteed to be a DAG}
		\caption{Greedy Forward-Backward Search}
		\label{algo: gfbs}
	\end{algorithm}

	In Algorithm \ref{algo: gfbs}, we outline a general framework based on greedy forward-backward search to learn a DAG $W$ by attempting to minimize the score $S(W)$. For now, we focus on the algorithm itself, and defer discussions of its soundness to Sections~\ref{sec: theoretical_guarantees}-\ref{sec: sample_complexity}. We denote this algorithm by GFBS for short.
	Crucially, in contrast to traditional greedy algorithms for structure learning, GFBS is \emph{vertex-greedy}: Instead of greedily adding edges to $\gr$, GFBS greedily adds \emph{vertices} to first build up a topological sort $T$ of $\gr$.
	Specifically, Line $4$ in the algorithm greedily finds the next vertex $i$ to add to the ordering, by comparing the score changes if we set the parents of $i$ to be the vertices already in the ordering. Conceptually, this step is one of the most important differences from GES which adds edges one at a time. We make this distinction clear in \cref{subsec: comparison_to_ges}.

	It is worth emphasizing that the output of GFBS is guaranteed to be a DAG.
	The backward phase is standard in greedy optimization, e.g. Greedy Equivalence Search (GES), and serves to eliminate unnecessary edges.
	In practice, in the backward phase, we could also process the edges in batches. As we explore in \cref{sec: sample_complexity}, in certain cases, this allows us to prove sample complexity upper bounds.

	\paragraph{Computational complexity}
	The running time of GFBS is a polynomial in $d$ and the time needed to compute the scores $S_i(\cdot)$. %
	More specifically, GFBS requires $O(d^{2})$ score evaluations (compared to $O(d2^{d})$ for exact algorithms). Evidently, a key computational concern is the complexity of evaluating the score in the first place. For many models such as linear, generalized linear, and exponential family models, this computation can be carried out in $\poly(n,d)$ time, which implies that GFBS on the whole runs in polynomial time. For nonparametric models, this computation may no longer be polynomial-time, but the total number of score evaluations is still $O(d^{2})$. In particular, GFBS always enjoys an exponential speedup over exact algorithms.

	\paragraph{Comparison to GES} In the supplement (see \cref{subsec: comparison_to_ges}), we exhibit linear Gaussian SEMs and illustrate how GES differs from GFBS for the least squares score as well as the traditional Gaussian BIC score. We first examine a folklore model where we show that their outputs sometimes differ. We also exhibit a model where they always differ. The key takeaway is that GFBS really is a distinct algorithm from GES.

	\subsection{Connection to equal variance SEM}
	\label{sec:eqvar}

	An important line of work starting with \citep{ghoshal2017ident} has shown that the assumption of equal variances in a linear Gaussian SEM \citep{peters2013} leads directly to an efficient, order-based algorithm. A similar idea in the setting of so-called quadratic variance function (QVF) DAGs was explored in \citep{park2017}.
	In this section, we show that the equal-variance algorithm of \citep{chen2018causal}, Algorithm~1, is a special case of GFBS.

	Define a score function as follows:
	\begin{align}
		\label{eq:def:lsscore}
		\ls(\gr)
		= \sum_{i=1}^{d} \E\var(X_{i}\given \pa_{\gr}(i)).
	\end{align}
	A few comments on this score function are in order:
	\begin{enumerate}
		\item The only assumption needed on $X$ for this score to be well-defined is that $\E XX^{T}$ is well-defined, i.e. $X_{i}\in L^{2}$ for each $i$.
		\item When $X$ satisfies a linear structural equation model $X=W^{T}X+z$, minimizing $\ls$ is equivalent to minimizing the least-squares loss $\sum_{i=1}^{d} (X_{i}-\ip{W^{(i)},X})^{2}$. \citet{loh2014causal} have shown that when $\cov(z)=\sigma^{2}I$, the unique global minimizer of the least-squares loss is the so-called \emph{equal variance SEM}.
		\item More generally, for nonlinear models, we have
		\begin{align}
			\label{eq:def:nlscore}
			\ls(\gr)
			&= \min_{g_1,\ldots,g_d\sim W}\sum_{i=1}^{d}\E(X_{i}-g_{i}(X))^{2},
		\end{align}
		where $g_1,\ldots,g_d\sim W$ indicates that for each $i$, $g_i$ depends only on the variables in $\pa_{\gr}(i)$. In other words, the minimum is taken over all functions $g_1,\ldots, g_d$ that respect the dependency structure implied by $W$. In this case, $g_i$ is essentially $\E[X_i\given\pa_W(i)]$.
		\item We can use \eqref{eq:def:nlscore} to define an empirical score in the obvious way given i.i.d. samples. Alternatively, the residual variance $\E\var(X_{i}\given \pa_{\gr}(i))$ can be replaced with any estimator of the residual variance.
	\end{enumerate}
	The GFBS algorithm consists of two phases: A forward phase and a backward phase. Our claim is that the forward phase of GFBS is identical to the equal-variance algorithm from  \citep{chen2018causal}:
	\begin{proposition}\label{prop: gfbs_eqvar}
		After the forward phase, the ordering $T$ returned by GFBS (Algorithm~\ref{algo: gfbs}) is the same as the ordering returned by the top-down equal-variance algorithm from \citet{chen2018causal}.
	\end{proposition}

	\begin{corollary}\label{cor: gfbs_global}
		Assume the linear SEM $X=W^{T}X+z$ with $\cov(z)=\sigma^{2}I$ under the score function \eqref{eq:def:lsscore}.
		Then GFBS returns a global minimizer of the problem \eqref{eq:scorebased}.
	\end{corollary}

	\cref{prop: gfbs_eqvar} will immediately follow from a more general statement which we prove in \cref{lem: main_thm}.
	An intriguing question is to what extent this observation extends to \emph{nonlinear} models such as additive noise models: While we do not have a proof, our experiments in Section~\ref{sec: expt} suggest something along these lines is true.

	\section{Bregman scores and identifiability via Bregman information}
	\label{sec: theoretical_guarantees}

	Motivated by the connection between GFBS, global optimality, and the least squares loss, in this section we establish a nice connection between the greedy algorithm and exponential families via the well-known duality between Bregman divergences and partition functions in exponential families \citep{banerjee2005clustering}. This can then be used to prove identifiability and recovery guarantees for GFBS.

	\paragraph{Bregman divergences and information} Let $\phi: \RR \to \RR$ be a strictly convex, differentiable function. Let $d_{\phi}(x, y) = \phi(x) - \phi(y) - (x - y) \phi' (y)$ be the Bregman divergence associated with $\phi$ and let $I_{\phi}(\calD) = \EE_{x \sim \calD}[d_{\phi}(x, \mu)]$ be the associated Bregman information. The Bregman-divergence is a general notion of distance that generalizes squared Euclidean distance, logistic loss, Itakuro-Saito distance, KL-divergence, Mahalanobis distance and generalized I-divergences, among others \cite{banerjee2005clustering}. The Bregman-information of a distribution is a measure of randomness of the distribution, that's associated with $\phi$. Among others, it generalizes the variance, the mutual-information and the Jensen-Shannon divergence of Gaussian processes \cite{banerjee2005clustering}. See \cref{subsec: bregman_div} for a brief review of this material and a basic treatment of Legendre duality, which will be used in the next section.

	\subsection{Bregman score functions, duality, and exponential families}\label{subsec: identifiability_and_recovery_via_gfbs}

	By replacing the least squares loss in \eqref{eq:scorebased} with a Bregman divergence $d_{\phi}$, we obtain the following score function, which we call a \emph{Bregman score}:
	\begin{align}
		\label{eq:globalscore}
		\bs(\gr)
		= \sum_{i} \E_X d_{\phi}(X_{i},\E[X_{i}\given \pa_W(i)])
		= \sum_{i} \min_{g_1,\ldots,g_d\sim\gr}\E_{X} d_{\phi}(X_{i},g_{i}(\pa_W(i)))
	\end{align}
	Before we study the behaviour of GFBS on Bregman scores, it is worth taking a moment to interpret this score function. To accomplish this, let us define the notion of an exponential random family DAG:
	\begin{definition}
		A DAG $\gr$ and a distribution $\calD$ define an exponential random family (ERF) DAG if (a) $\calD$ is Markov with respect to $\gr$, and (b) The local conditional probabilities come from an exponential family, i.e. $\pr(X_{i}\given \pa_{\gr}(i))\sim\ERF(g_{i},\psi_{i})$, where $\psi_{i}$ is the log-partition function of an exponential family with mean function $g_{i}(\pa_{\gr}(i))$.
	\end{definition}

	Since $\ERF(g_{i},\psi_{i})$ parametrizes a conditional distribution, its mean parameter $g_{i}$ is a function instead of vector, which explains our choice of notation.
	By the Markov property, any choice of local exponential family $\ERF(g_{i},\psi_{i})$ gives a well-defined joint distribution. The following lemma makes explicit the relationship between Bregman scores, exponential family DAGs, and the Bregman information. Let $\phi^*$ denote the Legendre dual of $\phi$.

	\begin{lemma}\label{lem: negative_log_likelihood}
		Let $\phi$ be a strictly convex, differentiable function and let $\psi:=\phi^{*}$. Then
		\begin{align}
			\label{eq: lem: negative_log_likelihood}
			\bs(\gr)
			&= \sum_{i \le d} \EE[I_{\phi}(X_i | \pa_W(i))]
			= -\sum_{i\le d} \E_{X}\log p_{g_{i},\psi}(X_{i}\given \pa_W(i)) - C(X)
		\end{align}
		where $C(X)$ depends only on $X$ and not the underlying DAG $\gr$ and $p_{g_{i},\psi}$ is the density of an $\ERF(g_i,\psi)$ model.
	\end{lemma}

	The proof of this lemma, which can be found in \cref{sec: proof_of_nll}, follows from the well-known correspondence between Bregman divergences and exponential families, given by the dual map $\phi\mapsto \phi^*$: Given a Bregman divergence $\phi$, there is a corresponding exponential family whose log-partition function is given by $\phi^{*}$ \citep{banerjee2005clustering} and vice versa.

	Importantly, \cref{lem: negative_log_likelihood} shows that the Bregman score $\bs$ is equivalent to the expected negative log-likelihood of an exponential family DAG whose local conditional probabilities all have the same log-partition function $\psi$. This means that minimizing the Bregman score can be naturally associated to maximizing the expected log likelihood of such a model. Similar observations had also been made and used in prior works on PCA \cite{collins2001generalization}, clustering \cite{banerjee2005clustering} and learning theory \cite{forster2002relative}.

	\subsection{Identifiability via Bregman information}\label{subsec: main_identifiability_theorem}

	Motivated by the connection between exponential family DAGs with the same local log-partition maps, in this section, we state our main assumption that generalizes the equal variance assumptions from prior works.

	First, we will need a mild assumption on $W$ that's of similar flavor to causal minimality, but with respect to the Bregman-information we are looking at.
    Denote $\mathcal{A}_W(i)$ to be the non-descendants of $i$ in the graph $W$,

	\begin{assumption}\label{assumption: causal_minimality}
		For all $i\le d$ and all subsets $Y \subseteq \mathcal{A}_W(i)$ such that $\pa(i) \not\subseteq Y$, $\EE[I_{\phi}(X_i | Y)] > \EE[I_{\phi}(X_i | \pa(i))]$.
	\end{assumption}

	This assumption essentially asserts that no edge in $W$ is superfluous with respect to the distribution on $X$. Now, we state our main assumption.

	\begin{assumption}[Equal Bregman-information upon conditioning]\label{assumption: eqinf}
		Assume that for a constant $\tau> 0$,
		\[\EE[I_{\phi}(X_i|\pa(i))] = \EE_w[I_{\phi}(X_i|\pa(i) = w)] = \tau\text{ for all $i \le n$}\]
		where $\pa(i)$ are the parents of $i$ in the underlying DAG $W$.
	\end{assumption}

	\begin{example}[Special case of  ANMs]
		Suppose we are working with an ANM. That is, there is a DAG $W$ such that for all $i \le d$, $X_i = f_i(\pa(i)) + \eps_i$ for some function $f_i$, where $\eps_i$ are jointly independent noise variables. Then, the above assumption says that there is a constant $\tau \ge 0$ such that for all $i$, $I_{\phi}(\eps_i) = \tau$. When $\phi(x) = x^2$, this is the well-known equal variance assumption.
	\end{example}

	We are now ready to state our main theorem.

	\begin{theorem}\label{lem: main_thm}
		Consider a model satisfying \cref{assumption: causal_minimality} and \cref{assumption: eqinf}. Under the Bregman score $\bs(\gr)$,
		the GFBS algorithm with tolerance parameter $\gamma = 0$ will output the true model.
	\end{theorem}

	As stated, the theorem holds for the population setting. The case of finite samples is studied in detail in \cref{sec: sample_complexity}, where we prove the same result given sufficient samples.

	\begin{corollary}\label{cor: main_identifiability}
		A model satisfying \cref{assumption: causal_minimality} and the Equal Bregman-information \cref{assumption: eqinf} is identifiable.
	\end{corollary}

	We defer the proof of the main theorem to the supplement, where we prove it for an even more general class of functionals that subsume the Bregman-information. Here, we make the following remarks regarding this proof.

	\begin{enumerate}
		\item The proof is actually shown for general functionals for which "conditioning drops value". Therefore, we don't need to only work with Bregman-information and we can instead work with many uncertainty measures of distributions that have this property. This is useful, for example, to show that non-Bregman-type models such as the QVF model from \citep{park2017} are identifiable using our framework.
        As a result, Theorem~\ref{lem: main_thm} subsumes several known identifiability results such as EQVAR~\citep{peters2013,chen2018causal}, NPVAR~\citep{gao2020polynomial}, QVF-ODS~\citep{park2017}, and GHD~\citep{park2019identifiability}. See Appendix~\ref{subsec: proof_of_main} for details.
		\item A similar proof could be adapted for other functionals of distributions that measure the randomness or uncertainty of the distribution. One class of examples could be generalized entropies \cite{amigo2018brief} such as the Shannon entropy, R\'{e}nyi entropy or the Tsallis entropy. We leave this for future work.
	\end{enumerate}

	\begin{remk}
		An important reason why our algorithm is efficient is because in line $4$ of Algorithm \ref{algo: gfbs}, we only compute a single score for each vertex not in the ordering so far. This works especially nicely with the Bregman score, precisely because conditioning with respect to more variables only lowers the Bregman information of a variable, as is exploited to prove the theorems above.
	\end{remk}

	\paragraph{A natural score function for non-parametric multiplicative models} We study multiplicative noise models of the form $X_i = f(\pa(i)) \eps_i$ from the perspective of the framework built so far. Examples of such models include growth models from economics and biology \cite{marshall2007life}.
	More specifically, we choose $\phi(x) = -\log x$ for which the Bregman divergence $d_{\phi}$ is the \textit{Itakuro-Saito} distance commonly used in the Signal and Speech processing community. The associated Bregman score is the \textit{Itakuro-Saito} score given by
	\[
	S_{\phi}(W)
	= \sum_{i \le d} (\EE\log \EE[X_i | \pa(i)] - \EE[\log X_i]).
	\]
	Interestingly, the equal Bregman-information assumption reduces purely to an assumption about the noise variables, akin to the equal variance assumption in the case of additive noise models. This suggests that for multiplicative models, the Itakuro-Saito score is naturally motivated from the perspective of identifiability. This gives a new insight into the applicability of score-based learning for multiplicative models, with theoretical foundations in our analysis. For details, see \cref{subsec: score_for_mult_models}.

	\section{Sample complexity}
	\label{sec: sample_complexity}
	To derive a sample complexity bound for GFBS, we first need to compute the Bregman score $\bs$; due to decomposability and \eqref{eq: lem: negative_log_likelihood}, this reduces to estimating the Bregman information $I_\phi$. Let the samples be denoted $(X_1^{(t)}, X_2^{(t)}, \ldots, X_d^{(t)})$ for $t = 1, 2, \ldots, n$. Denote the Bregman information of $X_i$ conditioning on a set $A$ with conditional mean plugged in as (after some calculation)
	\begin{align}
		S(X_i\given A)
		:= \E [I_\phi(X_i\given A)]
		= \E \phi(X_i) - \E \phi(\E (X_i\given A))
	\end{align}
	for some strictly convex, differentiable function $\phi$. To estimate this quantity, we can first apply nonparametric regression to estimate $f_{iA}: = \E (X_i\given A)$ and then take the sample mean:
	\begin{align}\label{eq:S_estimator}
		\widehat{S}(X_i\given A) = \frac{1}{n} \sum_{t \le n} \phi\big(X_i^{(t)}\big) - \frac{1}{n}\sum_{t \le n} \phi \big(\widehat{f}_{iA}(A^{(t)})\big).
	\end{align}
	To show convergence rate of this estimator, we will need some some regularity conditions on $f_{iA}$ and $\phi$. These assumptions are standard in the nonparametric statistics literature, see e.g., \cite[Chapters 1, 3]{gyorfi2002distribution}. First, we recall the definition of the H\"{o}lder class of functions:
	\begin{definition}\label{defn:holder}
		For any $r=(r_1,\cdots,r_d)$, $r_i\in \mathbb{N}$, let $|r|=\sum_i r_i$ and $D^r=\frac{\partial^{|r|}}{\partial x_1^{r_1}\cdots \partial x_d^{r_d}}$. The H\"{o}lder class $\Sigma(s,L)$ is the set of functions satisfying
		\[
		|D^r f(x) - D^r f(y)| \le L\|x-y\|^{s-r}\ \ \ \
		\]
		for all $r$ such that $|r|\le s$ and $x,y\in \mathbb{R}^d$.
	\end{definition}
	\begin{assumption}\label{asmp:regularity}
		Suppose for all $i$ and ancestor sets $A$ of $i$, $f_{iA} \in \Sigma(s,L)$. And suppose $\phi(X_i)$, $\phi(f_{iA})$ and $\phi'(f_{iA})$ all have finite second moments.
	\end{assumption}

	Denote $\mathcal{A}_W(i)$ to be the non-descendants of $i$ in graph $W$, then the following lemma says that we have a uniform estimator for the Bregman score:
	\begin{lemma}\label{lem:S_estimator_rate}
		Suppose the Bregman score and the conditional expectations satisfy Assumption~\ref{asmp:regularity}. Using the estimator defined in \eqref{eq:S_estimator} yields
		\[
		\min_{i\in[d], A\subseteq \mathcal{A}_W(i)}\prob\Big(\big| \widehat{S}(X_i\given A) - S(X_i\given A)\big| \le t\Big) \ge 1 - \frac{\delta^2_n}{t^2}
		\]
		where $\delta^2_n = C(n^{\frac{-2s}{2s+d}} + n^{-1})$ for some constant $C$.
	\end{lemma}
	Using this estimator, we can bound the sample complexity of the forward pass of GFBS as follows:
	\begin{theorem}[Forward phase of GFBS]\label{thm:samcom:gfbs:forward}
		Suppose the BN satisfies the identifiability condition in Theorem~\ref{lem: main_thm} and assumptions in Lemma~\ref{lem:S_estimator_rate}, denote the gap
		\[
		\Delta := \min_{\substack{i\in[d],A\subseteq \mathcal{A}_W(i)\\
				\pa(i)\not\subseteq A}}S(X_i\given A) - \tau > 0
		\]

		Let the ordering returned by the first phase of GFBS to be $\widehat{\pi} = (\widehat{\pi}_1,\cdots,\widehat{\pi}_d)$. If the sample size
		\[
		n \gtrsim \bigg(\frac{d^2}{\Delta^2 \epsilon}\bigg)^{\frac{2s + d}{2s} \vee 1}
		\]
		then $\prob(\widehat{\pi}\text{ is a valid ordering}) \ge 1 - \epsilon$.
	\end{theorem}

	The causal minimality \cref{assumption: causal_minimality} is equivalent to stating $\Delta > 0$. Theorem~\ref{thm:samcom:gfbs:forward} shows that $\Delta$ in fact controls the hardness of the estimation, which is the gap between the minimum Bregman information when all parents are conditioned on and when some parents are missing.

	In this section, to obtain strong bounds on sample complexity, we modify the backward phase of GFBS to be as follows:

	\begin{definition}\label{defn:gfbs:backward}
		Let $\widehat{A}_0=\emptyset$ and for $j \ge 1$, $\widehat{A}_j = \{\widehat{\pi}_i | i = 1, 2, \ldots j\}$. For each $\widehat{\pi}_{j+1}$, we find its parents from $\widehat{A}_j$ in the following way, estimate $S(X_{\widehat{\pi}_{j+1}}\given \widehat{A}_j)$ and $S(X_{\widehat{\pi}_{j+1}}\given \widehat{A}_j \setminus i)$ for $i\in\widehat{A}_j$. Then, set
		\begin{align}
			\label{eq:defn:gfbs:backward}
			\widehat{\pa}(\widehat{\pi}_{j+1})  = \widehat{A}_j \setminus \bigg\{i \in \widehat{A}_j \bigg| |\widehat{S}(X_{\widehat{\pi}_{j+1}}\given \widehat{A}_j) - \widehat{S}(X_{\widehat{\pi}_{j+1}}\given \widehat{A}_j \setminus i)|\le \gamma \bigg\}.
		\end{align}
	\end{definition}

	This says that we keep an edge $(i, \widehat{\pi}_{j+1})$ depending on its influence on the local score at the vertex $\widehat{\pi}_{j+1}$. If the influence is low, then we discard that edge. For our analysis to work, we process these low-influence edges in batches grouped according to the vertices they are oriented towards. In contrast, Algorithm~\ref{algo: gfbs} did not batch the edges and simply processed them one at a time.

	\begin{theorem}[Backward phase of GFBS]\label{thm:samcom:gfbs:backward}
		Suppose the same conditions and sample size in Theorem~\ref{thm:samcom:gfbs:forward} holds, using the backward phase defined in \eqref{eq:defn:gfbs:backward} with $\gamma=\Delta/2$ guarantees $\prob(\widehat{W}=W) \ge 1-\epsilon$.
	\end{theorem}

	Proofs can be found in \cref{app:samcom} in the supplement.

	\section{Experiments}\label{sec: expt}
	We conduct experiments to show the performance of GFBS on optimizing the Bregman score. We compare GFBS with existing score-based DAG learning algorithms: Gobnilp \citep{cussens2012}, NOTEARS \citep{zheng2020learning}, and GDS \citep{peters2013}. The implementation of these algorithms and data generating process are detailed in Appendix~\ref{app:expt}. Although previous works have evaluated the structure learning performance of special cases of GFBS such as equal variances, we also include these comparisons in the appendix for completeness.
	Also, in Appendix~\ref{app:unequal}, we investigate the performance of GFBS on models which violate the identifiability Assumption~\ref{assumption: eqinf}.

	\begin{itemize}
		\item \textit{Choice of $\phi$}. To show the generality of the Bregman score \eqref{eq:globalscore}, we investigate two convex functions to define the score: $\phi_1(x)=x^2$ and $\phi_2(x)=-\log x$. They correspond to sum of residual variances and sum of residual Itakuro-Saito (IS) distances respectively.
		\item \textit{Graph type}. We generate three types of graphs: Markov chains (MC), Erd\"os-R\'enyi (ER) graphs, Scale-Free (SF) graphs with different expected number of edges. We let the expected number of edges scale with $d$, e.g. ER-2 stands for Erd\"os-R\'enyi with $2d$ edges.
		\item \textit{Model type}. We simulate the data as $X_i=f_i(\pa(i))+ Z_i$ or $X_i=f_i(\pa(i))\times Z_i$ for different $\phi$'s, where $Z_i$ is independently sampled from some distribution such that Assumption~\ref{assumption: eqinf} is satisfied. Then we consider the following forms of the parental functions $f_i$: linear (LIN), sine (SIN), additive Gaussian process (AGP), and non-additive Gaussian process (NGP).
	\end{itemize}

	The main objective of these experiments is to evaluate the performance of these algorithms in optimizing the score: For this, it is necessary to compute the globally optimal score as a benchmark, which is computationally intensive. As a result, our experiments are restricted to $d=5$. We use Gobnilp \citep{cussens2012} to compute the global minimizer. The results are shown in Figure~\ref{fig:main}. As expected, GFBS returns a near-globally optimal solution in most cases when the sample size is large. Due to finite-sample errors, in some cases (notably on the IS score), GFBS returns a slightly higher score due to the backward phase, which allows the score to increase slightly in favour of sparser solutions. At a technical level, the issue is that the score does not distinguish I-maps from minimal I-maps, and this is exacerbated on finite samples. Better regularization and parameter tuning should resolve this, which we leave to future work. Nonetheless, the close alignment between GFBS and the globally optimal score suggest that GFBS---and hence the equal variance algorithm---is truly minimizing the score.

	\begin{figure}[h]
		\centering
		\includegraphics[width=\linewidth]{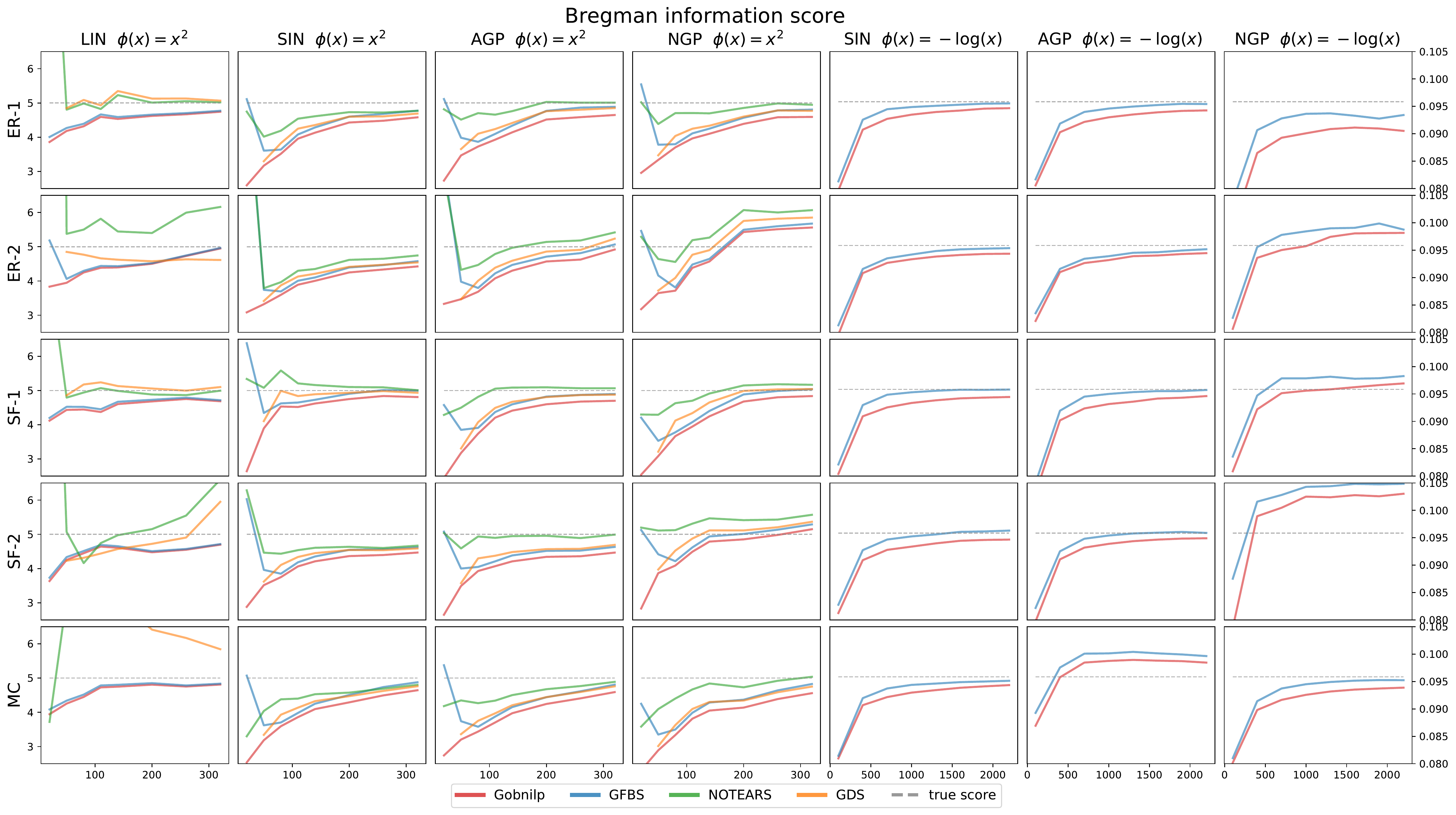}
		\caption{Score of output DAG vs. sample size $n$ for GFBS and 3 other algorithms. Left four columns: $\phi_1  (x)=x^2$ and $Z_i$ is $t$-distribution with variance $1$; Right three columns: $\phi_2(x)=-\log(x)$ and $Z_i$ is uniform distribution in $[1,2]$. The two sets of columns have different Y-axis scales. The grey dashed line is the score of the true graph.}
		\label{fig:main}
	\end{figure}

	\section{Discussion}

	We introduced the generic GFBS (Greedy Forward-Backward Search) algorithm for score-based DAG learning. It enjoys the guarantees of always outputting a DAG and running in time polynomial in the input size and the time required to compute the score function. We also showed statistical and sample complexity bounds for this algorithm for the generic Bregman score. We motivate this score by formally connecting it to the negative log-likelihood for all exponential DAG models, and considering the well-known approximation capabilities of exponential families, we expect that the Bregman score and our theoretical results apply to a wide variety of settings. In particular, the Bregman score generalizes the least squares score. For least-squares score, our sample complexity results unify and match or improve existing results such as \cite{ghoshal2017ident, chen2018causal, gao2020polynomial}. For generic Bregman scores, no sample complexity results were known prior to this work to the best of our knowledge and we provide the first such results.

	The GFBS algorithm also generalizes several prior works on greedy order-based algorithms for DAG learning, e.g., \cite{chen2018causal, gao2020polynomial, park2017, park2019identifiability}. Existing score-based greedy algorithms (such as GES or hill climbing) are edge-based, whereas these recent order-based algorithms are vertex-based. GFBS shows that each of these prior works can be re-interpreted as score-based greedy algorithms, each of which optimizes a different score.
    This brings them back under the umbrella of score-based learning. In our statistical guarantees, our assumptions generalize the equal variance assumption that has been studied in the literature in the last decade. Moreover, as a byproduct of our work, we also propose a new score function, the Itakuro-Saito score, for multiplicative SEM models and we leave it to future work to further explore the properties of this score function.

	For other future work, it would be insightful to compare \cref{assumption: causal_minimality} to the standard notions such as causal minimality. Moreover, our experiments suggest that the various assumptions we make are not strictly necessary, so an interesting future direction is to study weaker conditions under which GFBS globally optimizes the score.

	\subsection*{Broader impacts}
	Learning graphical models has important applications in causal inference, which is useful for mitigating bias in ML models.
	At the same time, causal models can be easily misinterpreted and provide a false sense of security, especially when they are subject to finite-sample errors.
	One additional potential negative impact from this line of work is the environmental cost of training large causal models, which can be expensive and time-consuming.

	\subsection*{Acknowledgements}
	We thank anonymous reviewers for their helpful comments in improving the manuscript.
	G.R. was partially supported by NSF grant CCF-1816372.
    B.K. was partially supported by advisor L\'aszl\'o Babai's NSF grant  CCF 1718902.
    B.A. was supported by NSF IIS-1956330, NIH R01GM140467, and the Robert H. Topel Faculty Research Fund at the University of Chicago Booth School of Business.
    All statements made are solely due to the authors and have not been endorsed by the NSF.

	\bibliography{greedydag}

\begin{thebibliography}{59}
\providecommand{\natexlab}[1]{#1}
\providecommand{\url}[1]{\texttt{#1}}
\expandafter\ifx\csname urlstyle\endcsname\relax
  \providecommand{\doi}[1]{doi: #1}\else
  \providecommand{\doi}{doi: \begingroup \urlstyle{rm}\Url}\fi

\bibitem[Amari(1995)]{amari1995information}
S.-i. Amari.
\newblock Information geometry of the em and em algorithms for neural networks.
\newblock \emph{Neural networks}, 8\penalty0 (9):\penalty0 1379--1408, 1995.

\bibitem[Amig{\'o} et~al.(2018)Amig{\'o}, Balogh, and
  Hern{\'a}ndez]{amigo2018brief}
J.~M. Amig{\'o}, S.~G. Balogh, and S.~Hern{\'a}ndez.
\newblock A brief review of generalized entropies.
\newblock \emph{Entropy}, 20\penalty0 (11):\penalty0 813, 2018.

\bibitem[Aragam et~al.(2015)Aragam, Amini, and Zhou]{aragam2015highdimdag}
B.~Aragam, A.~A. Amini, and Q.~Zhou.
\newblock Learning directed acyclic graphs with penalized neighbourhood
  regression.
\newblock \emph{arXiv:1511.08963}, 2015.

\bibitem[Aragam et~al.(2019)Aragam, Amini, and Zhou]{aragam2019globally}
B.~Aragam, A.~Amini, and Q.~Zhou.
\newblock Globally optimal score-based learning of directed acyclic graphs in
  high-dimensions.
\newblock 2019.

\bibitem[Banerjee et~al.(2005)Banerjee, Merugu, Dhillon, Ghosh, and
  Lafferty]{banerjee2005clustering}
A.~Banerjee, S.~Merugu, I.~S. Dhillon, J.~Ghosh, and J.~Lafferty.
\newblock Clustering with bregman divergences.
\newblock \emph{Journal of machine learning research}, 6\penalty0 (10), 2005.

\bibitem[Banerjee et~al.(2008)Banerjee, El~Ghaoui, and
  d'Aspremont]{banerjee2008}
O.~Banerjee, L.~El~Ghaoui, and A.~d'Aspremont.
\newblock Model selection through sparse maximum likelihood estimation for
  multivariate {G}aussian or binary data.
\newblock \emph{Journal of Machine Learning Research}, 9:\penalty0 485--516,
  2008.

\bibitem[Barndorff-Nielsen(2014)]{barndorff2014information}
O.~Barndorff-Nielsen.
\newblock \emph{Information and exponential families: in statistical theory}.
\newblock John Wiley \& Sons, 2014.

\bibitem[Bresler(2015)]{bresler2015efficiently}
G.~Bresler.
\newblock Efficiently learning ising models on arbitrary graphs.
\newblock In \emph{Proceedings of the forty-seventh annual ACM symposium on
  Theory of computing}, pages 771--782, 2015.

\bibitem[B{\"u}hlmann et~al.(2014)B{\"u}hlmann, Peters, and
  Ernest]{buhlmann2014}
P.~B{\"u}hlmann, J.~Peters, and J.~Ernest.
\newblock {CAM}: Causal additive models, high-dimensional order search and
  penalized regression.
\newblock \emph{Annals of Statistics}, 42\penalty0 (6):\penalty0 2526--2556,
  2014.

\bibitem[Chen et~al.(2019)Chen, Drton, and Wang]{chen2018causal}
W.~Chen, M.~Drton, and Y.~S. Wang.
\newblock {On causal discovery with an equal-variance assumption}.
\newblock \emph{Biometrika}, 106\penalty0 (4):\penalty0 973--980, 09 2019.
\newblock ISSN 0006-3444.
\newblock \doi{10.1093/biomet/asz049}.

\bibitem[Chickering et~al.(1995)Chickering, Geiger, and
  Heckerman]{chickering1995learning}
D.~Chickering, D.~Geiger, and D.~Heckerman.
\newblock Learning bayesian networks: Search methods and experimental results.
\newblock In \emph{proceedings of fifth conference on artificial intelligence
  and statistics}, pages 112--128, 1995.

\bibitem[Chickering(1996)]{chickering1996}
D.~M. Chickering.
\newblock Learning {B}ayesian networks is {NP}-complete.
\newblock In \emph{Learning from data}, pages 121--130. Springer, 1996.

\bibitem[Chickering(2003)]{chickering2003}
D.~M. Chickering.
\newblock Optimal structure identification with greedy search.
\newblock \emph{Journal of Machine Learning Research}, 3:\penalty0 507--554,
  2003.

\bibitem[Chickering et~al.(2004)Chickering, Heckerman, and
  Meek]{chickering2004}
D.~M. Chickering, D.~Heckerman, and C.~Meek.
\newblock Large-sample learning of {B}ayesian networks is {NP}-hard.
\newblock \emph{Journal of Machine Learning Research}, 5:\penalty0 1287--1330,
  2004.

\bibitem[Collins et~al.(2001)Collins, Dasgupta, and
  Schapire]{collins2001generalization}
M.~Collins, S.~Dasgupta, and R.~E. Schapire.
\newblock A generalization of principal components analysis to the exponential
  family.
\newblock In \emph{Nips}, volume~13, page~23, 2001.

\bibitem[Cussens(2012)]{cussens2012}
J.~Cussens.
\newblock Bayesian network learning with cutting planes.
\newblock \emph{arXiv preprint arXiv:1202.3713}, 2012.

\bibitem[Cussens et~al.(2017)Cussens, Haws, and Studen{\`y}]{cussens2017}
J.~Cussens, D.~Haws, and M.~Studen{\`y}.
\newblock Polyhedral aspects of score equivalence in bayesian network structure
  learning.
\newblock \emph{Mathematical Programming}, 164\penalty0 (1-2):\penalty0
  285--324, 2017.

\bibitem[Forster and Warmuth(2002)]{forster2002relative}
J.~Forster and M.~K. Warmuth.
\newblock Relative expected instantaneous loss bounds.
\newblock \emph{Journal of Computer and System Sciences}, 64\penalty0
  (1):\penalty0 76--102, 2002.

\bibitem[Gao et~al.(2020)Gao, Ding, and Aragam]{gao2020polynomial}
M.~Gao, Y.~Ding, and B.~Aragam.
\newblock A polynomial-time algorithm for learning nonparametric causal graphs.
\newblock \emph{arXiv preprint arXiv:2006.11970}, 2020.

\bibitem[Geiger and Heckerman(2002)]{geiger2002}
D.~Geiger and D.~Heckerman.
\newblock Parameter priors for directed acyclic graphical models and the
  characterization of several probability distributions.
\newblock \emph{Annals of Statistics}, 30:\penalty0 1412--1440, 2002.

\bibitem[Geiger et~al.(2001)Geiger, Heckerman, King, and Meek]{geiger2001}
D.~Geiger, D.~Heckerman, H.~King, and C.~Meek.
\newblock Stratified exponential families: {G}raphical models and model
  selection.
\newblock \emph{Annals of Statistics}, pages 505--529, 2001.

\bibitem[Ghoshal and Honorio(2017)]{ghoshal2017ident}
A.~Ghoshal and J.~Honorio.
\newblock Learning identifiable gaussian bayesian networks in polynomial time
  and sample complexity.
\newblock In \emph{Advances in Neural Information Processing Systems 30}, pages
  6457--6466. 2017.

\bibitem[Ghoshal and Honorio(2018)]{ghoshal2017sem}
A.~Ghoshal and J.~Honorio.
\newblock Learning linear structural equation models in polynomial time and
  sample complexity.
\newblock In A.~Storkey and F.~Perez-Cruz, editors, \emph{Proceedings of the
  Twenty-First International Conference on Artificial Intelligence and
  Statistics}, volume~84 of \emph{Proceedings of Machine Learning Research},
  pages 1466--1475, Playa Blanca, Lanzarote, Canary Islands, 09--11 Apr 2018.
  PMLR.

\bibitem[Glymour et~al.(2019)Glymour, Zhang, and Spirtes]{glymour2019review}
C.~Glymour, K.~Zhang, and P.~Spirtes.
\newblock Review of causal discovery methods based on graphical models.
\newblock \emph{Frontiers in genetics}, 10:\penalty0 524, 2019.

\bibitem[Gy{\"o}rfi et~al.(2002)Gy{\"o}rfi, Kohler, Krzy{\.z}ak, and
  Walk]{gyorfi2002distribution}
L.~Gy{\"o}rfi, M.~Kohler, A.~Krzy{\.z}ak, and H.~Walk.
\newblock \emph{A distribution-free theory of nonparametric regression},
  volume~1.
\newblock Springer, 2002.

\bibitem[Heckerman et~al.(1995)Heckerman, Geiger, and
  Chickering]{heckerman1995}
D.~Heckerman, D.~Geiger, and D.~M. Chickering.
\newblock Learning {B}ayesian networks: {T}he combination of knowledge and
  statistical data.
\newblock \emph{Machine learning}, 20\penalty0 (3):\penalty0 197--243, 1995.

\bibitem[Jalali et~al.(2011)Jalali, Johnson, and Ravikumar]{jalali2011learning}
A.~Jalali, C.~Johnson, and P.~Ravikumar.
\newblock On learning discrete graphical models using greedy methods.
\newblock \emph{arXiv preprint arXiv:1107.3258}, 2011.

\bibitem[Johnson et~al.(2012)Johnson, Jalali, and Ravikumar]{johnson2012high}
C.~Johnson, A.~Jalali, and P.~Ravikumar.
\newblock High-dimensional sparse inverse covariance estimation using greedy
  methods.
\newblock In \emph{Artificial Intelligence and Statistics}, pages 574--582.
  PMLR, 2012.

\bibitem[Liu et~al.(2014)Liu, Ye, and Fujimaki]{liu2014forward}
J.~Liu, J.~Ye, and R.~Fujimaki.
\newblock Forward-backward greedy algorithms for general convex smooth
  functions over a cardinality constraint.
\newblock In \emph{International Conference on Machine Learning}, pages
  503--511. PMLR, 2014.

\bibitem[Loh and B{\"u}hlmann(2014)]{loh2014causal}
P.-L. Loh and P.~B{\"u}hlmann.
\newblock High-dimensional learning of linear causal networks via inverse
  covariance estimation.
\newblock \emph{Journal of Machine Learning Research}, 15:\penalty0 3065--3105,
  2014.

\bibitem[Lokhov et~al.(2018)Lokhov, Vuffray, Misra, and
  Chertkov]{lokhov2018optimal}
A.~Y. Lokhov, M.~Vuffray, S.~Misra, and M.~Chertkov.
\newblock Optimal structure and parameter learning of ising models.
\newblock \emph{Science advances}, 4\penalty0 (3):\penalty0 e1700791, 2018.

\bibitem[Marshall and Olkin(2007)]{marshall2007life}
A.~W. Marshall and I.~Olkin.
\newblock \emph{Life distributions}, volume~13.
\newblock Springer, 2007.

\bibitem[Misra et~al.(2020)Misra, Vuffray, and Lokhov]{misra2020information}
S.~Misra, M.~Vuffray, and A.~Y. Lokhov.
\newblock Information theoretic optimal learning of gaussian graphical models.
\newblock In \emph{Conference on Learning Theory}, pages 2888--2909. PMLR,
  2020.

\bibitem[Nandy et~al.(2018)Nandy, Hauser, and Maathuis]{nandy2018}
P.~Nandy, A.~Hauser, and M.~H. Maathuis.
\newblock High-dimensional consistency in score-based and hybrid structure
  learning.
\newblock \emph{arXiv preprint arXiv:1507.02608}, 2018.

\bibitem[Ott and Miyano(2003)]{ott2003}
S.~Ott and S.~Miyano.
\newblock Finding optimal gene networks using biological constraints.
\newblock \emph{Genome Informatics}, 14:\penalty0 124--133, 2003.

\bibitem[Ott et~al.(2004)Ott, Imoto, and Miyano]{ott2004}
S.~Ott, S.~Imoto, and S.~Miyano.
\newblock Finding optimal models for small gene networks.
\newblock In \emph{Pacific symposium on biocomputing}, volume~9, pages
  557--567. Citeseer, 2004.

\bibitem[Park and Park(2019)]{park2019identifiability}
G.~Park and H.~Park.
\newblock Identifiability of generalized hypergeometric distribution (ghd)
  directed acyclic graphical models.
\newblock In \emph{The 22nd International Conference on Artificial Intelligence
  and Statistics}, pages 158--166. PMLR, 2019.

\bibitem[Park and Raskutti(2017)]{park2017}
G.~Park and G.~Raskutti.
\newblock {Learning quadratic variance function (QVF) dag models via
  overdispersion scoring (ODS)}.
\newblock \emph{The Journal of Machine Learning Research}, 18\penalty0
  (1):\penalty0 8300--8342, 2017.

\bibitem[Perrier et~al.(2008)Perrier, Imoto, and Miyano]{perrier2008}
E.~Perrier, S.~Imoto, and S.~Miyano.
\newblock Finding optimal bayesian network given a super-structure.
\newblock \emph{Journal of Machine Learning Research}, 9\penalty0
  (Oct):\penalty0 2251--2286, 2008.

\bibitem[Peters and B\"uhlmann(2013)]{peters2013}
J.~Peters and P.~B\"uhlmann.
\newblock Identifiability of {G}aussian structural equation models with equal
  error variances.
\newblock \emph{Biometrika}, 101\penalty0 (1):\penalty0 219--228, 2013.

\bibitem[Peters and B{\"u}hlmann(2014)]{peters2014identifiability}
J.~Peters and P.~B{\"u}hlmann.
\newblock Identifiability of gaussian structural equation models with equal
  error variances.
\newblock \emph{Biometrika}, 101\penalty0 (1):\penalty0 219--228, 2014.

\bibitem[Peters et~al.(2017)Peters, Janzing, and
  Sch{\"o}lkopf]{peters2017elements}
J.~Peters, D.~Janzing, and B.~Sch{\"o}lkopf.
\newblock \emph{Elements of causal inference: foundations and learning
  algorithms}.
\newblock MIT press, 2017.

\bibitem[Ramsey et~al.(2016)Ramsey, Glymour, Sanchez-Romero, and
  Glymour]{ramsey2016}
J.~Ramsey, M.~Glymour, R.~Sanchez-Romero, and C.~Glymour.
\newblock A million variables and more: the fast greedy equivalence search
  algorithm for learning high-dimensional graphical causal models, with an
  application to functional magnetic resonance images.
\newblock \emph{International Journal of Data Science and Analytics}, pages
  1--9, 2016.

\bibitem[Rao et~al.(2015)Rao, Shah, and Wright]{rao2015forward}
N.~Rao, P.~Shah, and S.~Wright.
\newblock Forward--backward greedy algorithms for atomic norm regularization.
\newblock \emph{IEEE Transactions on Signal Processing}, 63\penalty0
  (21):\penalty0 5798--5811, 2015.

\bibitem[Raskutti and Uhler(2018)]{raskutti2018}
G.~Raskutti and C.~Uhler.
\newblock Learning directed acyclic graph models based on sparsest
  permutations.
\newblock \emph{Stat}, 7\penalty0 (1), 2018.

\bibitem[Silander and Myllymaki(2006)]{silander2012}
T.~Silander and P.~Myllymaki.
\newblock A simple approach for finding the globally optimal bayesian network
  structure.
\newblock In \emph{Proceedings of the 22nd Conference on Uncertainty in
  Artificial Intelligence}, 2006.

\bibitem[Singh and Moore(2005)]{singh2005}
A.~P. Singh and A.~W. Moore.
\newblock Finding optimal bayesian networks by dynamic programming.
\newblock 2005.

\bibitem[Spirtes and Glymour(1991)]{spirtes1991}
P.~Spirtes and C.~Glymour.
\newblock An algorithm for fast recovery of sparse causal graphs.
\newblock \emph{Social Science Computer Review}, 9\penalty0 (1):\penalty0
  62--72, 1991.

\bibitem[Spirtes et~al.(2000)Spirtes, Glymour, and Scheines]{spirtes2000}
P.~Spirtes, C.~Glymour, and R.~Scheines.
\newblock \emph{Causation, prediction, and search}, volume~81.
\newblock The MIT Press, 2000.

\bibitem[Srebro(2003)]{srebro2003maximum}
N.~Srebro.
\newblock Maximum likelihood bounded tree-width markov networks.
\newblock \emph{Artificial intelligence}, 143\penalty0 (1):\penalty0 123--138,
  2003.

\bibitem[Teyssier and Koller(2005)]{teyssier2012}
M.~Teyssier and D.~Koller.
\newblock Ordering-based search: A simple and effective algorithm for learning
  bayesian networks.
\newblock In \emph{Uncertainty in Artifical Intelligence (UAI)}, 2005.

\bibitem[Tian et~al.(2016)Tian, Xu, and Gu]{tian2016forward}
L.~Tian, P.~Xu, and Q.~Gu.
\newblock Forward backward greedy algorithms for multi-task learning with
  faster rates.
\newblock In \emph{UAI}, 2016.

\bibitem[Tsamardinos et~al.(2006)Tsamardinos, Brown, and
  Aliferis]{tsamardinos2006}
I.~Tsamardinos, L.~E. Brown, and C.~F. Aliferis.
\newblock The max-min hill-climbing {B}ayesian network structure learning
  algorithm.
\newblock \emph{Machine Learning}, 65\penalty0 (1):\penalty0 31--78, 2006.

\bibitem[van~de Geer and B{\"u}hlmann(2013)]{geer2013}
S.~van~de Geer and P.~B{\"u}hlmann.
\newblock $\ell_0$-penalized maximum likelihood for sparse directed acyclic
  graphs.
\newblock \emph{Annals of Statistics}, 41\penalty0 (2):\penalty0 536--567,
  2013.

\bibitem[Vuffray et~al.(2016)Vuffray, Misra, Lokhov, and
  Chertkov]{vuffray2016interaction}
M.~Vuffray, S.~Misra, A.~Y. Lokhov, and M.~Chertkov.
\newblock Interaction screening: Efficient and sample-optimal learning of ising
  models.
\newblock \emph{arXiv preprint arXiv:1605.07252}, 2016.

\bibitem[Yuan and Malone(2013)]{yuan2013}
C.~Yuan and B.~Malone.
\newblock Learning optimal {B}ayesian networks: {A} shortest path perspective.
\newblock \emph{J. Artif. Intell. Res.(JAIR)}, 48:\penalty0 23--65, 2013.

\bibitem[Zhang(2008)]{zhang2008adaptive}
T.~Zhang.
\newblock Adaptive forward-backward greedy algorithm for sparse learning with
  linear models.
\newblock \emph{Advances in Neural Information Processing Systems},
  21:\penalty0 1921--1928, 2008.

\bibitem[Zheng et~al.(2018)Zheng, Aragam, Ravikumar, and Xing]{zheng2018dags}
X.~Zheng, B.~Aragam, P.~Ravikumar, and E.~P. Xing.
\newblock Dags with no tears: Continuous optimization for structure learning.
\newblock \emph{arXiv preprint arXiv:1803.01422}, 2018.

\bibitem[Zheng et~al.(2020)Zheng, Dan, Aragam, Ravikumar, and
  Xing]{zheng2020learning}
X.~Zheng, C.~Dan, B.~Aragam, P.~Ravikumar, and E.~Xing.
\newblock Learning sparse nonparametric dags.
\newblock In \emph{International Conference on Artificial Intelligence and
  Statistics}, pages 3414--3425. PMLR, 2020.

\end{thebibliography}
	\bibliographystyle{abbrvnat}

	\newpage
	\appendix

    {\Large\bf\centering Supplementary Material for ``Structure learning in polynomial time: Greedy algorithms, Bregman information, and exponential families''}

	\section{Comparison to GES}\label{subsec: comparison_to_ges}

	In order to compare GFBS to existing algorithms, in this appendix we present examples to compare the output of GFBS to GES.
	We first consider a model where this is some ambiguity in the outputs, and then exhibit a model where they always differ. The key takeaway is that GFBS really is a distinct algorithm from GES.

	\subsection{A setting where GES sometimes differs from GFBS}

	We will first consider the following standard example of a non-faithful distribution used in prior works \cite{peters2014identifiability} and show how GES differs from GFBS.

	Consider a distribution generated as $X_1 = N_1, X_2 = -X_1 + N_2, X_3 = X_1 + X_2 + N_3$ where $N_1, N_2, N_3$ are independent standard Gaussians. We will consider the score function
	\[S(W) = \sum_{i \le 3} (X_i - \sum_{j \in \pa(i)} W_{ji}X_j)^2\]
	to be minimized over all matrices $W$ whose support is a DAG.

	In the second forward step of GES, there are two equivalence classes that GES could have ended up with because they have the same scores, depending on how the tie is broken. One of them is the graph $X_1 \longrightarrow X_2 \longleftarrow X_3$ and the other is the graph $X_1 \longrightarrow X_2 \longrightarrow X_3$. If GES picked the former and continued with the algorithm, then it ultimately outputs the correct DAG. But if GES picked the latter which it very well could have, then it ends up outputting the wrong DAG $X_1\longrightarrow X_2 \longleftarrow X_3$ at the end of the algorithm.

	On the other hand, as shown in \cref{subsec: identifiability_and_recovery_via_gfbs} and \cref{sec: sample_complexity}, in both the population setting and the empirical setting for a reasonable sample size, GFBS will provably always output the correct DAG for this distribution since the residual variances are equal.

	We also considered the Gaussian BIC score that is traditionally used. In $100$ experiments under this score, GES fails all the time (also observed in prior works, for e.g. \cite{peters2014identifiability}) and outputs $X_1\longrightarrow X_2 \longleftarrow X_3$. But we note that GFBS succeeded in all $100$ experiments, although we do not give theoretical guarantees for this regularized score.

	\subsection{A setting where GES always differs from GFBS}

	We will tweak the weights of the model from the prior section and show that for this model, GES will always fail whereas GFBS will always succeed for essentially the same reason: Residual variances are equal.

	Consider a distribution generated as $X_1 = N_1, X_2 = -X_1 + N_2, X_3 = 0.9 X_1 + 0.9 X_2 + N_3$. We consider the same score function. We manually verify that GES will always output the DAG $X_1 \longrightarrow X_2 \longleftarrow X_3$ in the population setting. In $100$ experiments, GES also always outputted the same wrong DAG. Contrast this to GFBS which will always output the correct DAG in the population setting as well as the empirical setting with a reasonable number of samples.

	Finally, under the Gaussian BIC score, in $100$ experiments, GES always outputted the wrong DAG and GFBS always outputted the correct DAG, although we do not give theoretical guarantees in general for this phenomenon with the regularized score.

	\section{Bregman divergences, Bregman information and Legendre duality}\label{subsec: bregman_div}

	This set of definitions broadly follow the presentation of \cite{banerjee2005clustering}, but is specialized to our setting. Fix a strictly convex, differentiable function $\phi: \RR \to \RR$.

	\begin{definition}
		Define $d_{\phi}: \RR \times \RR \to \RR$ to be the Bregman-divergence of $\phi$ defined as
		\[d_{\phi}(x, y) = \phi(x) - \phi(y) - (x - y) \phi' (y)\]
		where $\phi'$ is the derivative of $\phi$.
	\end{definition}

	The Bregman-divergence is a general notion of distance that generalizes Squared Euclidean distance, Logistic Loss, Itakuro-Saito distance, KL-divergence, Mahalanobis distance and Generalized I-divergence, among others \cite{banerjee2005clustering}. In particular, it is nonnegative and is equal to $0$ if and only if the two arguments are equal.

	Of particular interest to us, in order to see how it connects to prior works on causal DAG learning, we illustrate with the following example that shows how the variance is a special case of the Bregman-information.

	\begin{example}\label{example: variance}
		Suppose $\phi(x) = x^2$. Then, $d_{\phi}(x, y) = (x- y)^2$ and $I_{\phi}(\calD) = \EE[(x - \EE[x])^2] = \var(\calD)$.
	\end{example}

	When we study multiplicative DAG models, we study another kind of Bregman-information that arises from the Itakura-Saito distance from signal processing theory. We explore this in more detail in \cref{subsec: score_for_mult_models}.

	\begin{example}\label{example: itakuro-saito}
		Assume the domain of the distribution and let $\phi : \RR^+ \to \RR$ be strictly convex. Suppose $\phi(x) = - \log x$. Then, $d_{\phi}(x, y) = \frac{x}{y} - \log\frac{x}{y} - 1$ and $I_{\phi}(\calD) = \EE[\frac{x}{\EE[x]} - \log\frac{x}{\EE[x]} - 1] = \log \EE[x] - \EE[\log x]$.
	\end{example}

	Any Bregman divergence defines a corresponding \emph{Bregman information}:
	\begin{definition}
		For a distribution $\calD$ over the reals, define the Bregman-information of $\phi$ as
		\[I_{\phi}(\calD) = \EE_{x \sim \calD}[d_{\phi}(x, \mu)]\]
		where $\mu = \EE_{x \sim \calD} [x]$ is the mean. For a random variable $X$ whose range is distributed as $\calD$ over $\RR$, we naturally define $I_{\phi}(X) := I_{\phi}(\calD)$
	\end{definition}

	The Bregman-information of a distribution is a measure of randomness of the distribution, that's associated with $\phi$. Among others, it generalizes the variance, the mutual-information and the Jensen-Shannon divergence of Gaussian processes \cite{banerjee2005clustering}.

	Bregman divergences have the following nice property that we will exploit in our analysis.

	\begin{proposition}[{\cite[Proposition ~1]{banerjee2005clustering}}]\label{propn: bregman_minimizer}
		The optimization problem
		\[ \min_{y \in \RR}\EE_{x \sim \calD} [d_{\phi}(x, y)]\]
		has a unique minimizer at $y = \EE_{x \sim \calD}[x]$.
	\end{proposition}

	We note the following:
	\begin{enumerate}
		\item \Cref{propn: bregman_minimizer} is surprising because $d_{\phi}(x, y)$ is convex with respect to the first argument $x$ but not necessarily with respect to the second argument $y$.
		\item Bregman-divergences are the only functionals with this property, i.e. the converse of \cref{propn: bregman_minimizer} is also true \cite[Appendix~B]{banerjee2005clustering}
		\item Bregman-information have many other nice properties that make them a useful analytic measure for studying randomness or uncertainty of distributions. See \cite{banerjee2005clustering} for details.
	\end{enumerate}

	Now, we briefly review the theory of Legendre duality that will be used in the sequel.

	\begin{definition}\label{def: dual_fn}
		For a function $\psi : \RR \to \RR$, define the dual function $\psi^*$ as
		\[\psi^*(t) = \sup_{\theta \in \RR} (t\theta - \psi(\theta))\]
	\end{definition}

	\begin{proposition}\label{prop: alt_def_of_dual}
		For a strictly convex, differentiable function $\psi : \RR \to \RR$,
		\[\psi^*(t) = tf(t) - \psi(f(t))\]
		where $f(t) = (\psi')^{-1}(t)$.
	\end{proposition}

	\begin{proof}
		Since $\psi$ is strictly convex and differentiable, $\psi'$ is monotonic and hence invertible, so $f$ is well-defined. Now, we can set the derivative of $t\theta - \psi(\theta)$ to zero to obtain that the maximizer $\theta^*$ in \cref{def: dual_fn} satsifies \[t = \psi'(\theta^*) \Longrightarrow \theta^* = f(t)\] Plugging this back in gives the result.
	\end{proof}

	We also note that when $\psi$ is strictly convex and differentiable, $\psi^*$ is also a strictly convex, differentiable function and $(\psi^*)^* = \psi$.

	An exponential random family (ERF) is a parametric family of distributions parametrized by the natural parameter $\theta$ with log partition function $\psi$ whose density is given by
	\[p_{(\psi, \theta)}(x) = \exp(x\theta - \psi(\theta))p_{0}(x)\]

	The log partition function $\psi$ must be strictly convex and differentiable. This is a general family of distributions that subsumes many standard families of parametric distributions such as the Gaussian distribution, the Poisson distribution and the Bernoulli distribution.

	Equivalently, the family could be parameterized by its expectation parameter $\mu = \EE_{x \sim p_{(\psi, \theta)}}[x]$

	\begin{fact}[\cite{barndorff2014information, amari1995information}]\label{fact: duality}
		For an ERF with natural parameter $\theta$, mean parameter $\mu$, log partition function $\psi$ and dual function $\phi = \psi^*$, we have the following duality:
		\[\mu = \psi'(\theta), \qquad \theta = \phi'(\mu)\]
		Therefore, $\phi'$ and $\psi'$ are inverses of each other.
	\end{fact}

	For a more general treatment of Legendre duality, see \cite{banerjee2005clustering}.

	\section{Proof of \cref{lem: negative_log_likelihood}} \label{sec: proof_of_nll}

	When the local conditional probability comes from an exponential family, $\pr(X_i | \pa(i)) \sim \ERF(\psi_i, g_i)$ with log-partition function $\psi_i$ and mean function $g_i(\pa(i))$, we can write the density as
	\begin{align}
		p_{g_{i},\psi_i}(X_{i}\given \pa(i)) = p_{(\psi_i, \theta_i)}(X_{i}\given \pa(i)) = \exp\Big\{X_{i}\theta_{i}(\pa(i)) - \psi(\theta_{i}(\pa(i)))\Big\}p_{0}(X_{i})
	\end{align}
	where $\theta_i(\pa(i))$ is the natural parameter corresponding to the mean parameter $g_i(\pa(i))$ associated with an ERF. Note that our notation is consistent with the notation from the previous section.

	We need to relate this expression to the Bregman-divergence. Towards that, we have the following lemma.

	\begin{lemma}\label{lem: div_to_prob}
		For an ERF with density $p_{(\psi, \theta)}$ with mean parameter $\mu$, we have
		\[d_{\phi}(x, \mu) = -\log p_{(\psi, \theta)}(x) + \phi(x) + \log p_0(x).\]
	\end{lemma}

	A similar result had been obtained in different contexts - PCA \cite{collins2001generalization}, clustering \cite{banerjee2005clustering} and learning theory \cite{forster2002relative}.
	A proof follows from essentially similar ideas but we include a proof here for completeness.

	\begin{proof}
		Firstly, using \cref{prop: alt_def_of_dual} and \cref{fact: duality}, we have
		\begin{align*}
			\phi(\mu) = \psi^*(\mu) &= \mu((\psi')^{-1}(\mu)) - \psi((\psi')^{-1}(\mu))\\
			&= \mu\phi'(\mu) - \psi(\phi'(\mu))\\
			&= \mu\theta - \psi(\theta).
		\end{align*}
		We also have
		\begin{align*}
			-\log p_{(\psi, \theta)}(x) &= \psi(\theta) - x\theta - \log p_0(x).
		\end{align*}
		Therefore,
		\begin{align*}
			d_{\phi}(x, \mu) &= \phi(x) - \phi(\mu) - (x - \mu)\phi'(\mu)\\
			&= \phi(x) - \phi(\mu) - (x - \mu)\theta\\
			&= \phi(x) - (\mu\theta - \psi(\theta)) - (x - \mu)\theta\\
			&= \phi(x) + \psi(\theta) - x\theta\\
			&= -\log p_{(\psi, \theta)}(x) + \phi(x) + \log p_0(x).
			\qedhere
		\end{align*}
	\end{proof}

	\cref{lem: negative_log_likelihood} now follows immediately.

	\begin{proof}[Proof of \cref{lem: negative_log_likelihood}]
		Using \cref{lem: div_to_prob},
		\begin{align*}
			S_{\phi}(W) &= \sum_i \EE_{\pa(i)} I_{\phi}(X_i | X_{\pa(i)})\\
			&= \sum_i \EE_{\pa(i)} \EE_{X_i}[d_{\phi}(X_i, \EE[X_i | \pa(i)])|\pa(i)]\\
			&= \sum_i \EE_{\pa(i)} \EE_{X_i}[d_{\phi}(X_i, g_i(\pa(i)))|\pa(i)]\\
			&= \sum_i \EE[d_{\phi}(X_i, g_i(\pa(i)))]\\
			&= \sum_i \EE[-\log p_{g_i, \psi}(X_i | \pa(i)) + \phi(X_i) + \log p_0(X_i)]\\
			&= -\sum_{i} \E_{X}\log p_{g_{i},\psi}(X_{i}\given \pa(i)) - C(X)
		\end{align*}
		where $C(X)$ depends only on $X$ and not the underlying DAG $\gr$.
	\end{proof}

	\section{Proof of \cref{lem: main_thm}}\label{subsec: proof_of_main}

	We will prove the result for a more general class of functionals $g$ that subsume the Bregman-information.

	For a function $f: \RR \times \RR \longrightarrow \RR$ and a distribution $\calD$ over $\RR$, define $\mu(f, \calD)$ to be a minimizer (fix an arbitrary choice) of $\EE_{x \sim \calD}[f(x, y)]$ over $y \in \RR$. That is, for all reals $y \in \RR$,
	\[\EE_{x \sim \calD}f(x, y) \ge \EE_{x \sim \calD}[f(x, \mu(f, \calD))]\]

	Fix $f$ and a random variable $X$ with distribution $\calD$. Define $g(X) = \EE_{X \sim \calD}[f(X, \mu(f, \calD))]$. The choice of $\mu$, among all possible minimizers, doesn't matter because the value of the functional $g$ will be the same for any such choice. For practical applications, we would generally want $g$ to be efficiently approximable via finite samples.

	\begin{lemma}\label{prop: gisinformation}
		Suppose $f= d_{\phi}$ for a strictly convex, differentiable $\phi: \RR \to \RR$. Then, $g$ is the Bregman-information $I_{\phi}$.
	\end{lemma}

	\begin{proof}
		Using \cref{propn: bregman_minimizer}, we obtain $\mu(f, \calD) = \EE_{x \sim \calD}[x]$. Therefore,
		\begin{align*}
			g(X) = \EE[f(X, \mu(f, \calD))] = \EE[d_{\phi}(X, \EE_{X \sim \calD}[X])] = I_{\phi}(X)
		\end{align*}
	\end{proof}

	Consider a distribution $X = (X_1, \ldots, X_d)$ with an underling DAG $W$. Suppose for all $i$, \[\EE[g(X_i | X_{\pa(i)})] = \EE_w[g(X_i|X_{\pa(i)} = w)] = \tau\]
	where $\tau$ is a constant. Note that this reduces to the the equal Bregman-information assumption \cref{assumption: eqinf} when $f = d_{\phi}$.

	We now prove the following generalization of a similar result by \cite{gao2020polynomial}.

	\begin{lemma}\label{lem: main_lemma1}
		Let $Y$ be a fixed set of variables. Then, for any $i$ such that $X_i \not\in Y$ and no element of $Y$ is a descendant of $X_i$,
		\begin{align*}
			\EE[g(X_i|Y)] &= \tau \text{ if $X_{\pa(i)} \subseteq Y$}\\
			\EE[g(X_i|Y)] &\ge \tau \text{ otherwise}
		\end{align*}
		Moreover, if for all ancestral sets $Y$ of $i$ such that $\pa(i) \not\subseteq Y$, we had $\EE[g(X_i | X_Y)] > \EE[g(X_i | X_{\pa(i)})]$, then the inequality above is strict.
	\end{lemma}

	\begin{proof}[Proof of \cref{lem: main_lemma1}]
		Let $\calD$ denote the marginal distribution of $X_i$ and let $\calD_{A}$ denote the marginal distribution of $X_i$ conditioned on fixing the variable $A$.

		If $X_{\pa(i)} \subseteq Y$, then
		\begin{align*}
			\EE[g(X_i | Y)] &= \EE[g(X_i|X_{\pa(i)}, Y \setminus X_{\pa(i)})]\\
			&= \EE[g(X_i|X_{\pa(i)})]\\
			&= \tau
		\end{align*}
		where we used the fact that conditioned on $X_{\pa(i)}$, $X_i$ is independent of $Y \setminus X_{\pa(i)}$.

		On the other hand, suppose $X_{\pa(i)} \not\subseteq Y$. Let $Z = X_{\pa(i)} \setminus Y$ be the set of free parent variables. For the sake of brevity, denote by $\mu_Y$ the quantity $\mu(f, \calD_Y)$ and denote by $\mu_{Y, Z}$ the quantity $\mu(f, \calD_{Y, Z})$. We have
		\begin{align*}
			\EE[g(X_i|Y)] &= \EE[\EE[f(X_i, \mu_Y)|Y]]\\
			&= \EE[\EE[ \EE[f(X_i, \mu_Y)| Y, Z] | Y]]\\
			&\ge \EE[\EE[ \EE[f(X_i, \mu_{Y, Z})| Y, Z] | Y]]\\
			&= \EE[\EE[ g(X_i |Y, Z) | Y]]\\
			&= \EE[g(X_i | Y, Z)]\\
			&= \tau
		\end{align*}
		where the inequality followed from the definition of $\mu$ and the last equality used the preceding case that we've already shown, since $X_{\pa(i)} \subseteq Y \cup Z$. Finally, if we had the condition $\EE[g(X_i | X_Y)] > \EE[g(X_i | X_{\pa(i)})]$ for all ancestral sets $Y$ not containing $pa(i)$, then the inequality in the display above also strictly holds because $Z \neq \emptyset$.
	\end{proof}

	We can now prove \cref{lem: main_thm}.

	\begin{proof}[Proof of \cref{lem: main_thm}]
		Let $f = d_{\phi}$. Then, we observe that $g = I_{\phi}$ by \cref{prop: gisinformation}. Therefore, \cref{lem: main_lemma1} can be applied in this setting, and by \cref{assumption: causal_minimality}, strict inequality holds.

		Consider the forward phase of GFBS. Let the vertices added to $T$ be $v_1, v_2, \ldots, v_d$ respectively in that order. We prove by strong induction on $i$ that for all $i \ge 1$, $v_i$ is a source node (a vertex of indegree $0$) of the graph $W \setminus\{v_1, \ldots, v_{i - 1}\}$.

		To prove the base case, observe that if a vertex $v$ has a parent in $W$, then $\EE[g(X_v)] > \EE[g(X_v | X_{\pa(v)})] = \tau$. On the other hand, if $v$ is any source node of the graph, then $\EE[g(X_v)] = \tau$. Hence, $v_1$ will be a source node of the graph, proving the base case.

		Assume the result holds for all indices upto $i$ and consider $v_{i+ 1}$. Let $H$ be the DAG $W \setminus \{v_1, \ldots, v_i\}$. Consider an arbitrary vertex $w$ in $H$. Firstly, because of the induction hypothesis, no $v_j$ is a descendant of $w$ for any $j \le i$. Now, if $w$ is a source node in $H$, that is, $\pa(w) \subseteq \{v_1, \ldots, v_i\}$, then $\EE[g(X_w | X_{v_1}, \ldots, V_{v_i})] = \EE[g(X_w | X_{\pa(w)}] = \tau$, where used the Markov property. On the other hand, if a vertex $w$ in $H$ has a parent in $H$, that is, $\pa(w) \subsetneq \{v_1, \ldots, v_i\}$, then $\EE[g(X_w | X_{v_1}, \ldots, V_{v_i})] > \EE[g(X_w | X_{\pa(w)}] = \tau$. These two assertions prove that $V_{i + 1}$ is a source node of $H$, proving the induction step.

		Therefore, for all $i \ge 1$, we must have $\pa(v_i) \subseteq \{v_1, \ldots, v_{i - 1}\}$. This proves that $T = v_1 \ldots, v_d$ is a topological sorting of the vertices of the graph. In the backward phase, all edges $e = (i, j)$ not in $W$ will be removed from $W$ because the score will not change after removing $e$ because $j$'s current parents will contain $\pa_W(j)$. Ultimately, the true DAG $W$ remains which will be returned by GFBS.
	\end{proof}

	We now explain how \cref{prop: gfbs_eqvar} follows from \cref{lem: main_thm}. By \cref{example: variance}, if we take $\phi(x) = x^2$, then the corresponding Bregman information is the variance. In \cite{chen2018causal}, they consider a linear SEM with equal error variances. The assumption of equal error variances is precisely the equal Bregman-information \cref{assumption: eqinf} we impose. Now, they iteratively find source nodes for the graph and then condition on them. But by the above inductive proof of \cref{lem: main_thm}, this is exactly what happens in the forward phase of GFBS. Therefore, GFBS recovers their algorithm when specialized to equal error-variance linear SEMs.

	Other examples of functionals $g$ from previous works include:
	\begin{itemize}
		\item NPVAR~\citep{gao2020polynomial}: $g(X_i\given A) = \var(X_i\given A)$
		\item QVF-ODS~\citep{park2017}: $g(X_i\given A) = \var(T_i(X_i)\given A) - \E(T_i(X_i)\given A)$ with $\tau=0$, $T_i$ is a linear transformation that depends on $\E(X_i \given A)$.
		\item GHD~\citep{park2019identifiability}: $g(X_i\given A) = ((X_i)_r) - \E f_i^{(r)}(\E(X_i\given A))$ with $\tau=0$, where $(a)_r = a(a-1)\cdots(a-r+1)$ and $f_i^{(r)}$ is a $r$th factorial constant moments ratio (CMR) function of form
		\[
		f_i^{(r)}(x;a(i),b(i))=x^r\prod_{k=1}^{p_i}\frac{(a_{ik}+r-1)_r}{a_{ik}^r}\prod_{\ell=1}^{q_i}\frac{b^r_{i\ell}}{(b_{i\ell}+r-1)_r}
		\]
		such that
		\[
		\E[(X_i)_r\given \pa(i)] = f_i^{(r)}\big(\E[X_i\given \pa(i)];a(i),b(i)\big)
		\]
		for any integer $r \le  \max X_i$.
	\end{itemize}

	\section{A natural score function for non-parametric multiplicative models}\label{subsec: score_for_mult_models}

	Consider the multiplicative SEM model
	\[X_i = f(X_{\pa(i)}) \eps_i\]
	with an underlying DAG $W$. We will also assume that $\eps_i$ is positive with probability $1$. Examples of such models include growth models from economics and biology \cite{marshall2007life}.

	Let $\phi(x) = -\log x$. Then, the Bregman divergence $d_{\phi}$ will be the Itakuro-Saito distance used in Signal and Speech processing community. From \cref{cor: main_identifiability}, we get that the model is identifiable under the condition
	\[\EE[I_{\phi}(X_i|X_{\pa(i)})] = \text{constant}.\]
	But we can compute this explicitly for a multiplicative model. Firstly, note that $\EE[X_i | X_{\pa(i)} = w] = \EE[f(X_{\pa(i)})\eps_i | X_{\pa(i) = w}] = f(X_{\pa(i)})\EE[\eps_i]$. Using the same calculations as in \cref{example: itakuro-saito}, we get
	\begin{align*}
		\EE[I_{\phi}(X_i|X_{\pa(i)})] &= \EE_w[\EE[-\log\frac{X_i}{\EE[X_i | X_{\pa(i)} = w]} | X_{\pa(i)} = w]]\\
		&= \EE_w [\EE[-\log\frac{\eps_i}{\EE[\eps_i]} | X_{\pa(i)} = w]]\\
		&= \log \EE[\eps_i] - \EE[\log \eps_i]
	\end{align*}

	Therefore, the equal Bregman-information assumption is equivalent to the following assumption on the noise variables
	\[\log \EE[\eps_i] - \EE[\log \eps_i] = \text{constant}\]

	This is satisfied for instance when $\eps_i$ are identically distributed. Our theory of Bregman scores illustrates that when such assumptions are feasible, such as in the case of identically distributed noise variables, then to estimate such models via score based approaches, a great candidate score would be the Itakuro-Saito score
	\[S_{\phi}(W) = \sum_{i \le d} \EE[I_{\phi}(X_i | X_{\pa(i)})] = \sum_{i \le d} (\EE\log \EE[X_i | X_{\pa(i)}] - \EE[\log X_i]).\]

	\section{Proofs for Section~\ref{sec: sample_complexity}}\label{app:samcom}
	\subsection{Proof of Lemma~\ref{lem:S_estimator_rate}}
	\begin{proof}
		For all $i\in [d]$ and $A\subseteq \mathcal{A}_G(i)$,
		\begin{align*}
			\E\bigg(\widehat{S}(X_i\given A) - S(X_i\given A)\bigg)^2 &\lesssim \E\bigg(\E\phi(X_i) - \frac{1}{n}\sum_t\phi(X_i^{(t)}) \bigg)^2 + \E\Bigg( \E\phi(f_{iA}) -  \frac{1}{n}\sum_t \phi \bigg(\widehat{f}_{iA}(A^{(t)})\bigg)\Bigg)^2 \\
			& \lesssim n^{-1} + \E\Bigg( \E\phi(f_{iA}) -  \frac{1}{n}\sum_t \phi \bigg(\widehat{f}_{iA}(A^{(t)})\bigg)\Bigg)^2
		\end{align*}
		due to the finite second moment and parametric rate. For the second term,
		\begin{align*}
			&\E\Bigg( \E\phi(f_{iA}) -  \frac{1}{n}\sum_t \phi \bigg(\widehat{f}_{iA}(A^{(t)})\bigg)\Bigg)^2 \\
			& = \E \Bigg(\E\phi(f_{iA}) - \frac{1}{n}\sum_t \phi \bigg(f_{iA}(A^{(t)})\bigg) + \frac{1}{n}\sum_t \phi \bigg(f_{iA}(A^{(t)})\bigg) - \frac{1}{n}\sum_t \phi \bigg(\widehat{f}_{iA}(A^{(t)})\bigg)\Bigg)^2 \\
			&\lesssim  \E \Bigg(\E\phi(f_{iA}) - \frac{1}{n}\sum_t \phi \bigg(f_{iA}(A^{(t)})\bigg)\Bigg)^2 + \frac{1}{n}\sum_t\E\Bigg( \phi \bigg(f_{iA}(A^{(t)})\bigg) - \phi \bigg(\widehat{f}_{iA}(A^{(t)})\bigg)\Bigg)^2\\
			& \lesssim n^{-1} + \frac{1}{n}\sum_t \E \bigg(\phi'(f_{tA}(A^{(t)}))\bigg)^2\E\bigg(f_{iA}(A^{(t)}) - \widehat{f}_{iA}(A^{(t)})\bigg)^2 \\
			& \lesssim  n^{-1} + n^{\frac{-2s}{2s+d}}
		\end{align*}
		For the first term, the inequality is by finite second moment and parametric rate. For second term, apply first order Taylor expansion and absorb the high order estimation error terms into the constant before the inequality. Finally, the tail probability bound follows by Markov's inequality.
	\end{proof}

	\subsection{Proof of Theorem~\ref{thm:samcom:gfbs:forward}}
	\begin{proof}
		Let $\widehat{A}_0=\emptyset$ and for $j \ge 1$, $\widehat{A}_j = \{\widehat{\pi}_i | i = 1, 2, \ldots j\}$. Denote the event $\mathcal{E}_j = \{\widehat{\pi}_{j} \text{ is a source node of } G[V\setminus \widehat{A}_{j-1}]\}$. Then
		\begin{align*}
			\prob(\widehat{\pi}\text{ is a valid ordering}) =  \prod_{j=0}^{d-1}\prob(\mathcal{E}_{j+1} \given \mathcal{E}_j)
		\end{align*}
		For each term of the product,
		\[
		\prob(\mathcal{E}_{j+1} \given \mathcal{E}_j) = \sum_{\substack{A\text{ is a subset of non-descendants}\\|A|=j}}\prob(\mathcal{E}_{j+1} \given \widehat{A}_j=A, \mathcal{E}_j)\prob(\widehat{A}_j=A\given \mathcal{E}_j)
		\]
		$\mathcal{E}_j$ implies that $\widehat{A}_j=A$ is of size $j$ and a subset of non-descendants of remaining nodes. More importantly, all possibilities sum up to one
		\[
		\sum_{\substack{A\text{ is a subset of non-descendants}\\|A|=j}}\prob(\widehat{A}_j=A\given \mathcal{E}_j)=1
		\]
		Invoking Lemma~\ref{lem:S_estimator_rate}, union bound the estimation error
		\[
		\prob\bigg(\cup_{i\notin V\setminus A}\bigg\{|\widehat{S}(X_i\given A) - S(X_i\given A)| \ge t \bigg\}\bigg) \le \sum_{i\notin V\setminus A}\prob\bigg(|\widehat{S}(X_i\given A) - S(X_i\given A)| \ge t\bigg) \le   (d-j)\frac{\delta^2_n}{t^2}
		\]
		Thus, with probability at least $1 - (d-j)\delta^2_n/t^2$, we have
		\[
		\begin{cases}
			\widehat{S}(X_i\given A) \le \tau + t & i\text{ is a source node of }G[V\setminus A] \\
			\widehat{S}(X_i\given A) \ge \tau + \Delta - t & i\text{ is not a source node of }G[V\setminus A]
		\end{cases}
		\]
		Therefore, with $t\le \Delta/2$, the node $\widehat{\pi}_{j+1}$ found by GFBS which minimizes the score is still a source node. This implies for all possible $A$,
		\[
		\prob(\mathcal{E}_{j+1} \given \widehat{A}_j=A, \mathcal{E}_j)\ge 1 - 4(d-j)\frac{\delta^2_n}{\Delta^2}
		\]
		furthermore,
		\[
		\prob(\mathcal{E}_{j+1} \given \mathcal{E}_j)\ge 1 - 4(d-j)\frac{\delta^2_n}{\Delta^2}
		\]
		and finally
		\[
		\prob(\widehat{\pi}\text{ is a valid ordering}) = \prod_j\prob(\mathcal{E}_{j+1} \given \mathcal{E}_j)\ge 1 - \sum_{j=0}^{d-1}(d-j)\frac{4\delta^2_n}{\Delta^2} \ge 1 - \frac{4d^2\delta^2_n}{\Delta^2}
		\]
		Solving $\prob(\widehat{\pi}\text{ is a valid ordering}) > 1-\epsilon$ yields the desired result.
	\end{proof}

	\subsection{Proof of Theorem~\ref{thm:samcom:gfbs:backward}}

	\begin{proof}
		We only need to show parents of each node are correctly estimated. Theorem~\ref{thm:samcom:gfbs:forward} guarantees that the parents of $\widehat{\pi}_{j+1}$ are in $\widehat{A}_j$. Thus,
		\[
		S(\widehat{\pi}_{j+1} \given \widehat{A}_j) =S(\widehat{\pi}_{j+1} \given \pa(\widehat{\pi}_{j+1}))
		\]
		By the definition of $\Delta$,
		\begin{align*}
			\begin{cases}
				S(X_{\widehat{\pi}_{j+1}}\given \widehat{A}_j\setminus i) - S(X_{\widehat{\pi}_{j+1}}\given \widehat{A}_j) \ge \Delta & i\in\pa(\widehat{\pi}_{j+1}) \\
				S(X_{\widehat{\pi}_{j+1}}\given \widehat{A}_j\setminus i) - S(X_{\widehat{\pi}_{j+1}}\given \widehat{A}_j) = 0 & i\not\in\pa(\widehat{\pi}_{j+1})
			\end{cases}
		\end{align*}
		Invoking Lemma~\ref{lem:S_estimator_rate}, with probability at least $1 - d\delta^2_n/t^2$, for all $i\in \widehat{A}_j \cup \{\emptyset\}$
		\[
		|\widehat{S}(X_{\widehat{\pi}_{j+1}}\given \widehat{A}_j \setminus i)) - S(X_{\widehat{\pi}_{j+1}} \given \widehat{A}_j \setminus i))| \le t
		\]
		which implies
		\[
		\begin{cases}
			|\widehat{S}(X_{\widehat{\pi}_{j+1}}\given \widehat{A}_j) - \widehat{S}(X_{\widehat{\pi}_{j+1}}\given \widehat{A}_j \setminus i)| \le 2t & i\in\pa(\widehat{\pi}_{j+1}) \\
			|\widehat{S}(X_{\widehat{\pi}_{j+1}}\given \widehat{A}_j) - \widehat{S}(X_{\widehat{\pi}_{j+1}}\given \widehat{A}_j \setminus i)| \ge \Delta - 2t & i\notin\pa(\widehat{\pi}_{j+1})
		\end{cases}
		\]
		With $2t \le \Delta/2=\gamma$, we can distinguish parents of $\widehat{\pi}_{j+1}$ from other non-descendants, thus $\widehat{\pa}(\widehat{\pi}_{j+1}) = \pa(\widehat{\pi}_{j+1})$. A union bound over $d$ nodes gives us the same sample complexity as in Theorem~\ref{thm:samcom:gfbs:forward}, which completes the proof.
	\end{proof}

	\section{Unequal Bregman score cases}\label{app:unequal}
	In this section, we investigate the behaviour of GFBS when the equal Bregman information condition is violated. As is evident from the proofs in Appendix~\ref{subsec: proof_of_main}, exact equality is actually not necessary for the proof and the algorithm to go through.
	The Assumption~\ref{assumption: eqinf} has a straightforward extension analogous in the literature \citep{ghoshal2017ident}, which also ensures identifiability. We present the result here, whose proof follows Appendix~\ref{subsec: proof_of_main} and previous work and thus is omitted.
	\begin{assumption}\label{assumption: uneqinf}
		There exists a valid ordering $\pi$ such that for all $i\in {[d]}$ and $\ell \in \pi_{[i+1:d]}$,
		\begin{align*}
			\E[I_\phi(X_i\given X_{\pi_{[1:i-1]}})] = \E[I_\phi(X_i\given \pa(i))] < \E[I_\phi(X_\ell \given X_{\pi_{[1:i-1]}})]
		\end{align*}
	\end{assumption}

	To demonstrate this assumption, we conduct experiments with $\phi(x)=x^2$, which leads to $\E[I_\phi(X_i\given \pa(i))] = \E\var(X_i\given \pa(i))$.
	We can generate data satisfying this ``unequal'' assumption for Markov chain + sine model + Gaussian noise. The idea is that to restrict the range of noise variance.

	Suppose the Markov chain is $X_1\to\cdots\to X_d$, with $X_i = \sin(X_{i-1}) + Z_i$, and $Z_i\sim \mathcal{N}(0,\sigma_i^2)$. We restrict the $\sigma_i^2$ to be sampled from $[1,1.2]$. To make sure Assumption~\ref{assumption: uneqinf} is satisfied, we need for any $i$ and $i < \ell \le d$,
	\[
	\sigma_i^2 = \E\var(X_i\given X_{i-1}) < \E\var(X_\ell\given X_{i-1}) = \E\var(\sin(X_{\ell-1} + \epsilon_{\ell-1})\given X_{i-1}) +\sigma^2_\ell
	\]
	It suffices to find a lower bound on $\var(\sin(a + X))$ for any $a\in[-1,1]$ and $X\sim\mathcal{N}(0,\sigma^2)$.
	\begin{lemma}\label{lem:unequal:sine_gaussian}
		Suppose $X\sim\mathcal{N}(0,\sigma^2)$ with $\sigma^2\ge 1$, then for any $a\in [-1,1]$, $\var(\sin(X+a)) \ge 1/4$.
	\end{lemma}
	With Lemma~\ref{lem:unequal:sine_gaussian}, we can show that the identifiability is guaranteed:
	\[
	\E\var(\sin(X_{\ell-1} + \epsilon_{\ell-1})\given X_{i-1}) \ge \E (1/4 \given X_{i-1}) =1/4 > 0.2 = \max(\sigma_i^2 - \sigma^2_\ell)
	\]
	\begin{proof}[Proof of Lemma~\ref{lem:unequal:sine_gaussian}]
		Using the identity
		\[
		\sin(X+a) = \sin a \cos X + \cos a \sin X
		\]
		we have
		\[
		\E \sin (X+a) = \sin a \E\cos X = \sin a \times \text{Re} [\E \exp(-i X)] = \sin a \exp(-\sigma^2/2).
		\]
		Moreover, a short calculation shows that
		\begin{align*}
			\E \sin^2(X+a)
			& = \frac{1}{2} - \frac{1}{2}\cos 2a \exp(-2\sigma^2).
		\end{align*}
		Finally,
		\begin{align*}
			\var\sin(X+a) &= \E \sin^2(X+a) - (\E \sin(X+a))^2 \\
			&= \frac{1}{2} - \frac{1}{2}e^{-2\sigma^2} + \sin^2a \times e^{-\sigma^2}(e^{-\sigma^2} - 1)\\
			&> 1/4.
			\qedhere
		\end{align*}
	\end{proof}
	To illustrate this condition, we run a simple experiment as follows:
	Consider two settings for $\sigma_i^2$: Sampled uniformly and randomly from (a) $[1,1.2]$ or (b) $[0.1,1.9]$. Run GFBS and Gobnilp on generated data, then compare the score obtained by two algorithms, and test whether the ordering of estimated graph is correct. Since the true graph is a Markov chain, there is only one true ordering.

	As shown in Figure~\ref{fig:unequal_expt}, when Assumption~\ref{assumption: uneqinf} is satisfied through Lemma~\ref{lem:unequal:sine_gaussian}, the score output by GFBS is close to the true one, and the topological ordering can be recovered. When the range of $\sigma^2_i$ is not well-controlled, GFBS does not return the correct ordering, and neither does Gobnilp. Interestingly, GFBS nonetheless does a good job at optimizing the score.
	\begin{figure}[h]
		\centering
		\begin{subfigure}[t]{0.75\textwidth}
			\includegraphics[width=1.\textwidth]{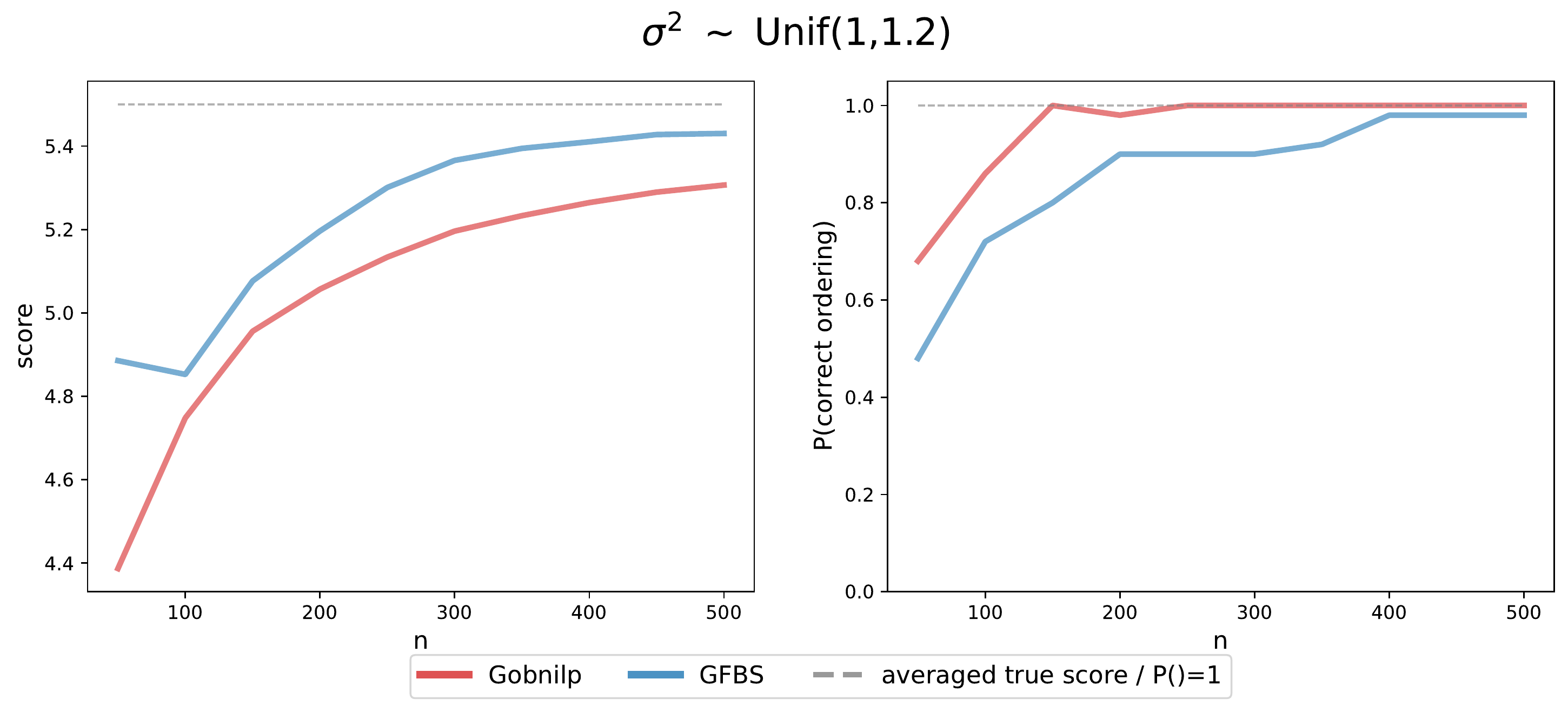}
		\end{subfigure}\\
		\begin{subfigure}[t]{0.75\textwidth}
			\includegraphics[width=1.\textwidth]{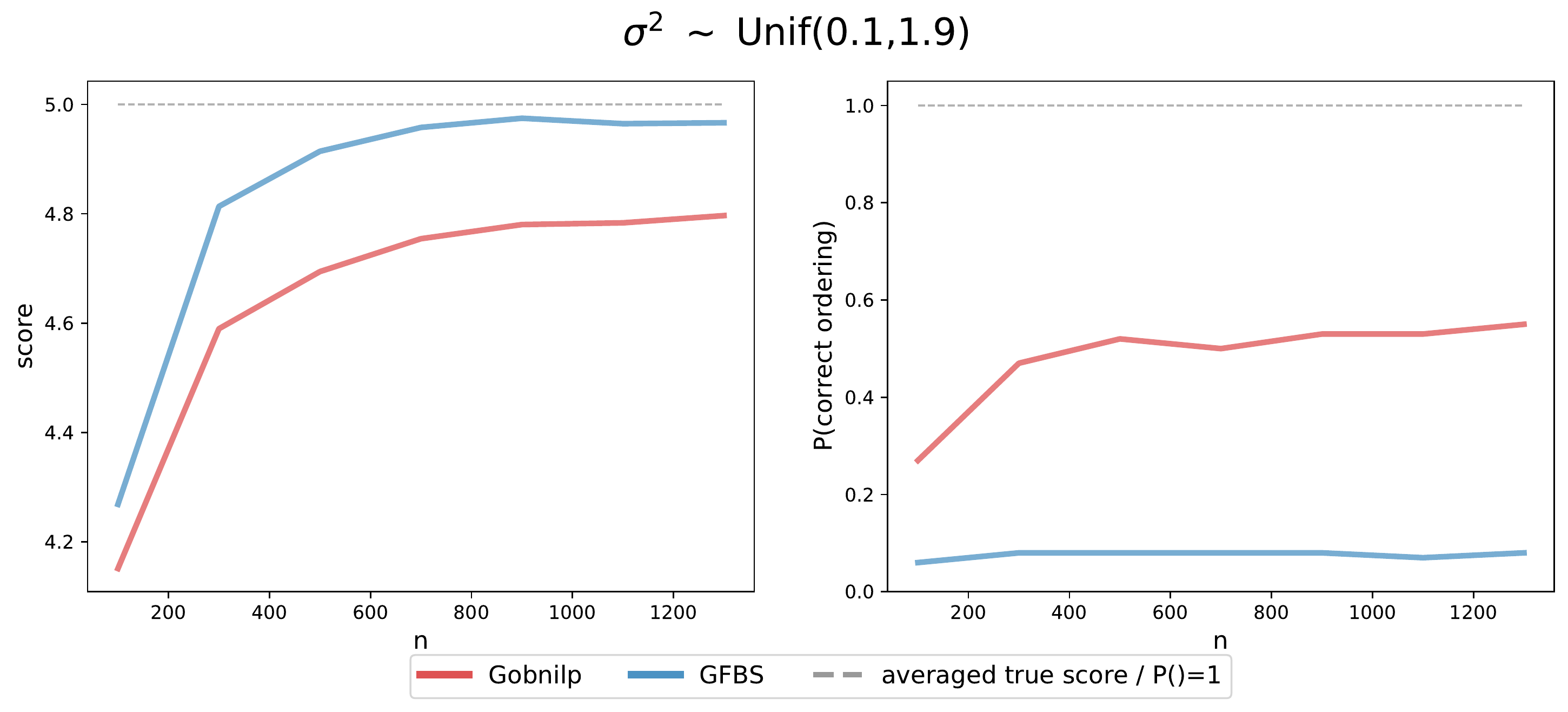}
		\end{subfigure}
		\caption{Unequal Bregman score experiments. Left column: score output by GFBS and Gobnilp; Right column: $\prob(\text{correct ordering})$ v.s. sample size; Upper row: range of $\sigma^2_i$ is $[1,1.2]$; Bottom row: range of $\sigma^2_i$ is from $[0.1,1.9]$. The gray dashed lines indicate average true score ($\sum_i\sigma^2_i$, left) or the optimal probability of recovery (right).}
		\label{fig:unequal_expt}
	\end{figure}

	\section{Experiment details}\label{app:expt}

	In this appendix we collect all the details of the experiments in Section~\ref{sec: expt}.

	\subsection{Experiment settings}
	\textbf{Bregman scores}: We define them through convex functions
	\begin{itemize}
		\item \textit{Residual variances}: $\phi_1(x)=x^2$
		\item \textit{Itakuro-Saito}: $\phi_2(x) = -\log (x)$
	\end{itemize}
	\textbf{Graph types}: We let the expected number of edges to scale with $d$, e.g. ER-2 stands for Erd\"os-R\'enyi with $2d$ edges.
	\begin{itemize}
		\item \textit{ER}: randomly choose $s$ edges from all possible $\binom{d}{2}$ directed edges, then randomly permute the nodes
		\item \textit{SF}: scale-free graphs generated through Barabasi-Albert process
		\item \textit{MC}: Markov chain, randomly permute the nodes
	\end{itemize}
	\textbf{Model types}: We specify the parental functions $f_i$ to be as follows: linear model (LIN), sine model (SIN), additive Gaussian Process (AGP) and non-additive Gaussian Process (NGP) with  with kernel $K(x,y) = e^{-\|x-y\|^2 / 2}$.
	\begin{itemize}
		\item For $\phi_1$, data is generated according to the form of $X_i=f_i(\pa(i))+Z_i$
		\begin{itemize}
			\item $\sigma$: additive noise standard deviation, set to 1
			\item Noise distribution:
			\begin{itemize}
				\item \textit{Gaussian}: $Z_i = \sigma \times \mathcal{N}(0,1)$
				\item \textit{t}: $Z_i = \sigma\times t(3) / \sqrt{3}$
				\item \textit{Gumbel}: $Z_i = \sigma \times Gumbel(0,\sqrt{6}/\pi)$
			\end{itemize}
			\item Parental functions:
			\begin{itemize}
				\item \textit{LIN}: $f_i = \sum_{\ell \in \pa(i)} \beta_{i\ell}X_\ell $, where $\beta_{i \ell} = Rademacher \times Unif(0.5,1.2) $
				\item \textit{SIN}: $f_i = \sum_{\ell \in \pa(i)} \sin(X_\ell) $
				\item \textit{AGP}: $f_i = \sum_{\ell \in \pa(i)} GP(X_\ell) $
				\item \textit{NGP}: $f_i =  GP(\pa(i)) $
			\end{itemize}
		\end{itemize}
		\item For $\phi_2$, data is generated according to the form of $X_i = f_i(\pa(i))\times Z_i$.
		\begin{itemize}
			\item Noise distribution:
			\begin{itemize}
				\item \textit{Uniform}: $Z_i\sim Unif(1,2)$
			\end{itemize}
			\item Parental function:
			\begin{itemize}
				\item \textit{SIN}: $f_i = \frac{1}{|\pa(i)|}\sum_{\ell\in\pa(i)}\sin^2 (X_\ell)$;
				\item \textit{AGP}: $f_i=\frac{1}{|\pa(i)|}\sum_{\ell\in\pa(i)}GP^2(X_\ell) + 0.5$;
				\item \textit{NGP}: $f_i=\frac{1}{2} GP^2(\pa(i)) + 0.5$
			\end{itemize}
		\end{itemize}
	\end{itemize}

	\textbf{Other parameters:}
	\begin{itemize}
		\item Dimension: $d=5,10,20,30$
		\item Number of edges: $s = kd$, $k=1,2,4$
		\item Sample size: $n=$[50,80,110,140,200,260,320] for $\phi_1$, $n=$[100,400,700,...,2200] for $\phi_2$
		\item Replications of simulation: $N=30$
	\end{itemize}

	\subsection{Implementation of algorithms}
	For GFBS, use Generalized Additive Model (GAM) to estimate all conditional expectations and compute all local scores as reference for all other methods. In particular, GAM is replaced by ordinary least square for \textit{LIN} model. In backward phase, use threshold $\gamma=0.05$ for $\phi_1$ and $\gamma=0.0005$ for $\phi_2$. GAM is implemented by Python package \texttt{pygam} with default parameters to avoid favoring one particular method due to hyper-parameter tuning,

	We compare GFBS with following score-based structure learning algorithms:
	\begin{itemize}
		\item Gobnilp\citep{cussens2012}: is an exact solver for score-based Baysian network learning through Constraint Integer Programming. We input the local scores output by GFBS for it to optimize. The implementation is available at \url{https://www.cs.york.ac.uk/aig/sw/gobnilp/}.
		\item NOTEARS\citep{zheng2018dags,zheng2020learning}: uses an algebraic characterization of DAGs for score-based structure learning of    nonparametric models via partial derivatives. We adopt example hyper-parameters to run, then compute the total score of output DAG using local scores output by GFBS. The implementation is available at \url{https://github.com/xunzheng/notears}.
		\item GDS\citep{peters2013}: greedily searches over neighbouring DAGs differed by adding / deleting / reversing one edge. Switch the score from log likelihood to our score setting, use \texttt{gam} function in R package \texttt{mgcv} with P-splines \texttt{bs=`ps'} and the default smoothing parameter \texttt{sp=0.6} to estimate the conditional expectations. In particular, GAM is replaced by ordinary least square for \textit{LIN} model. Implementation is available at \url{https://academic.oup.com/biomet/article/101/1/219/2364921#supplementary-data}. Omitted for $d> 10$ due to computational cost.
		\item GES\citep{chickering2003}: greedily searches over neighbouring Markov equivalence class to optimize the score. We use \texttt{sem-bic} score with \texttt{penaltyDiscount=0}, which amounts to score equals $BIC = 2L$ where $L$ is the likelihood. Only run for linear model and compare SHD. Implementation is available at \url{https://github.com/bd2kccd/py-causal}.
	\end{itemize}

	These simulations used an Intel E5-2680v4 2.4GHz CPU running on an internal cluster.

	\subsection{Evaluation metrics}
	\begin{itemize}
		\item Structural Hamming Distance (SHD): common metric for comparing performance in structure learning, which counts the total number of edge additions, deletions, and reversals needed to convert the estimated graph into the true graph.
		\item Bregman Score: $\sum_i\E\phi(X_i) - \E\phi(\E(X_i\given\pa(i)))$. Except for the true score indicated by grey dashed lines, this metric is evaluated in finite sample using estimator defined in \eqref{eq:S_estimator}.
	\end{itemize}

	\subsection{Additional experiments}
	Here we present some additional experiments.

	\subsubsection{Main figures with other noise}
	In Figure~\ref{fig:additional}, we present the left four columns of Figure~\ref{fig:main} under other two noise distribution: Gaussian and Gumbel. Note that the experiment settings are under $\phi_1$.

	\begin{figure}[h]
		\centering
		\includegraphics[width=\linewidth]{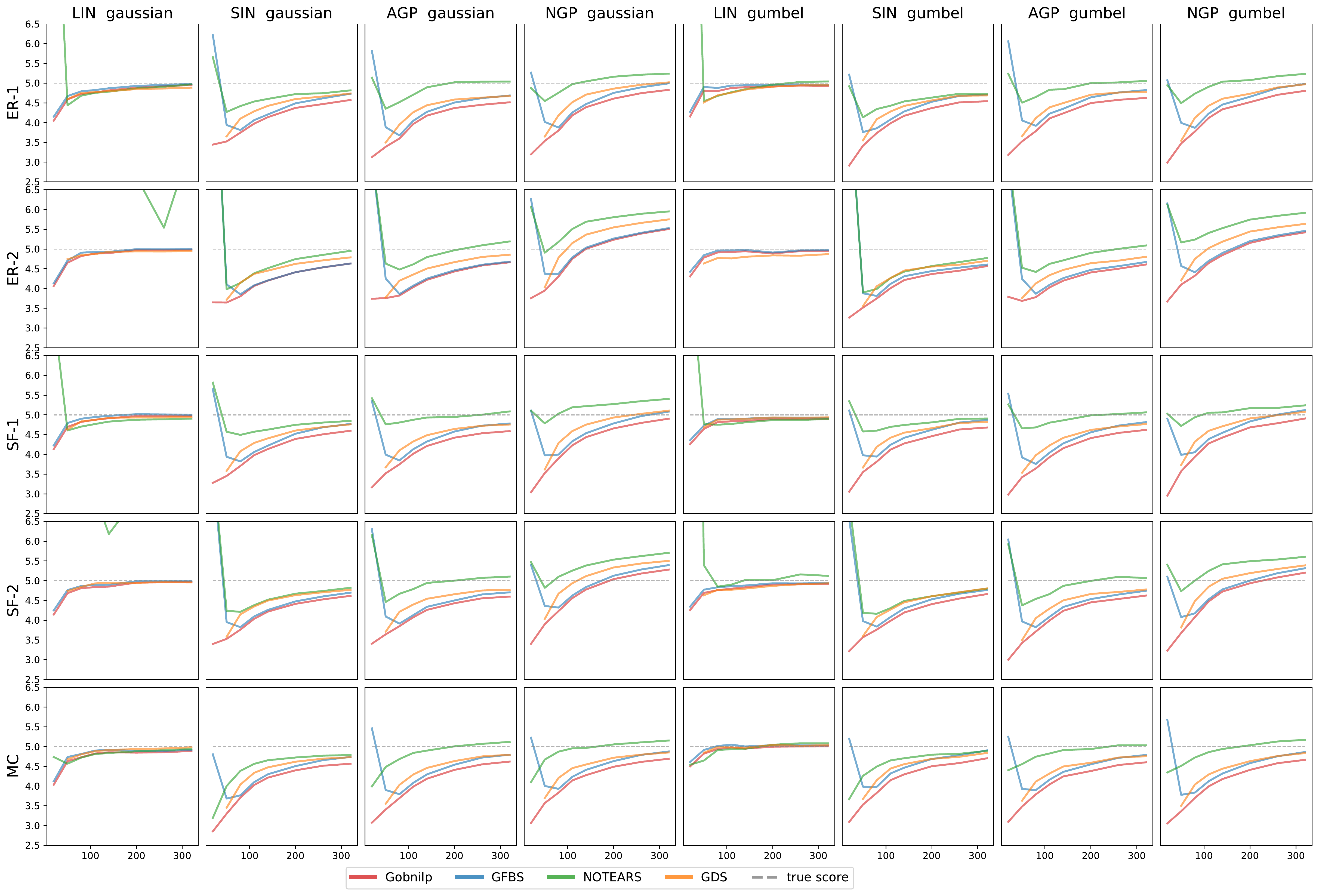}
		\caption{Score of output DAG vs. sample size $n$ for GFBS and 3 other algorithms for $\phi_1$ settings. Left four columns: $Z_i$ is Gaussian distribution with variance $1$; Right four columns: $Z_i$ is Gumbel distribution with variance $1$. The grey dashed line is the score of the true graph.}
		\label{fig:additional}
	\end{figure}

	\subsubsection{Structure learning}
	We illustrate the performance of GFBS on structural learning by considering experiment settings under higher dimensions and $\phi_1$, where Gobnilp and GDS are omitted due to heavy computational cost.
	\begin{itemize}
		\item Figure~\ref{fig:shd:gaussian}: SHD v.s. sample size $n$ for $d=20,30$, Gaussian noise, and $\phi_1$.
		\item Figure~\ref{fig:shd:t}: SHD v.s. sample size $n$ for $d=20,30$, t noise, and $\phi_1$.
		\item Figure~\ref{fig:shd:gumbel}: SHD v.s. sample size $n$ for $d=20,30$, Gumbel noise, and $\phi_1$.
	\end{itemize}

	\begin{figure}[h]
		\centering
		\begin{subfigure}[t]{0.49\textwidth}
			\includegraphics[width=1.\textwidth]{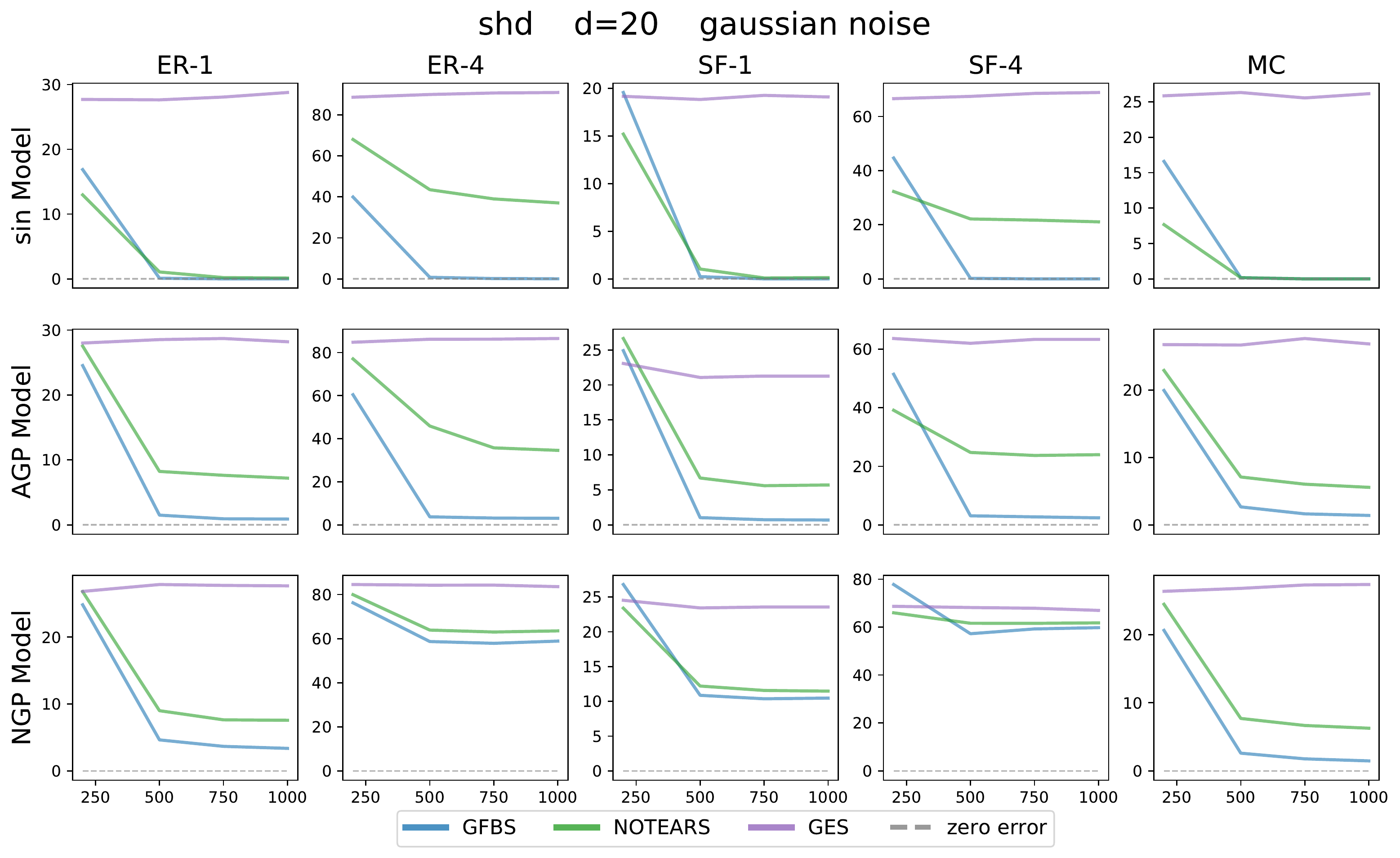}
		\end{subfigure}
		\begin{subfigure}[t]{0.49\textwidth}
			\includegraphics[width=1.\textwidth]{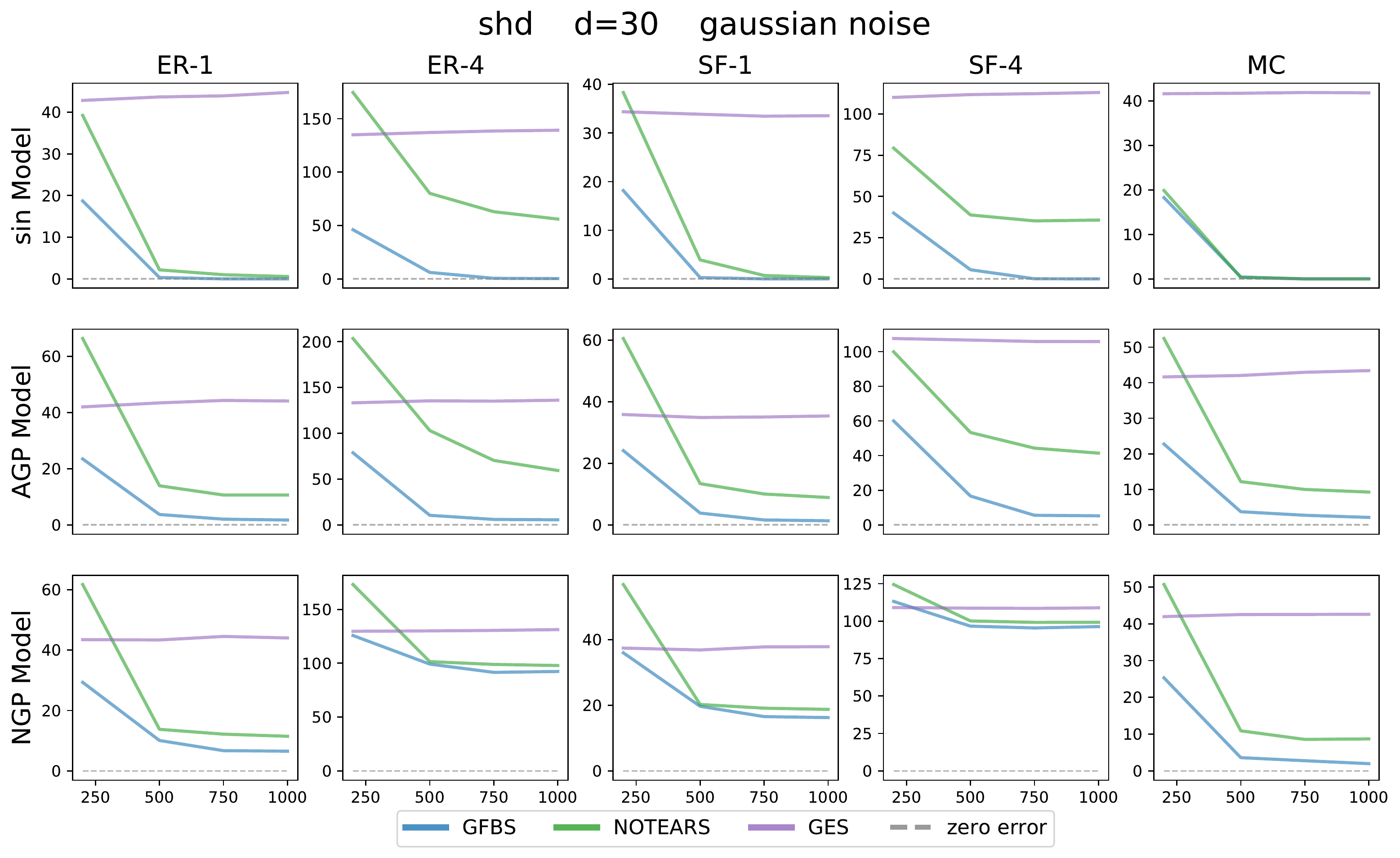}
		\end{subfigure}
		\caption{SHD v.s. sample size for $d=20,30$ and Gaussian noise.}
		\label{fig:shd:gaussian}
	\end{figure}

	\begin{figure}[h]
		\centering
		\begin{subfigure}[t]{0.49\textwidth}
			\includegraphics[width=1.\textwidth]{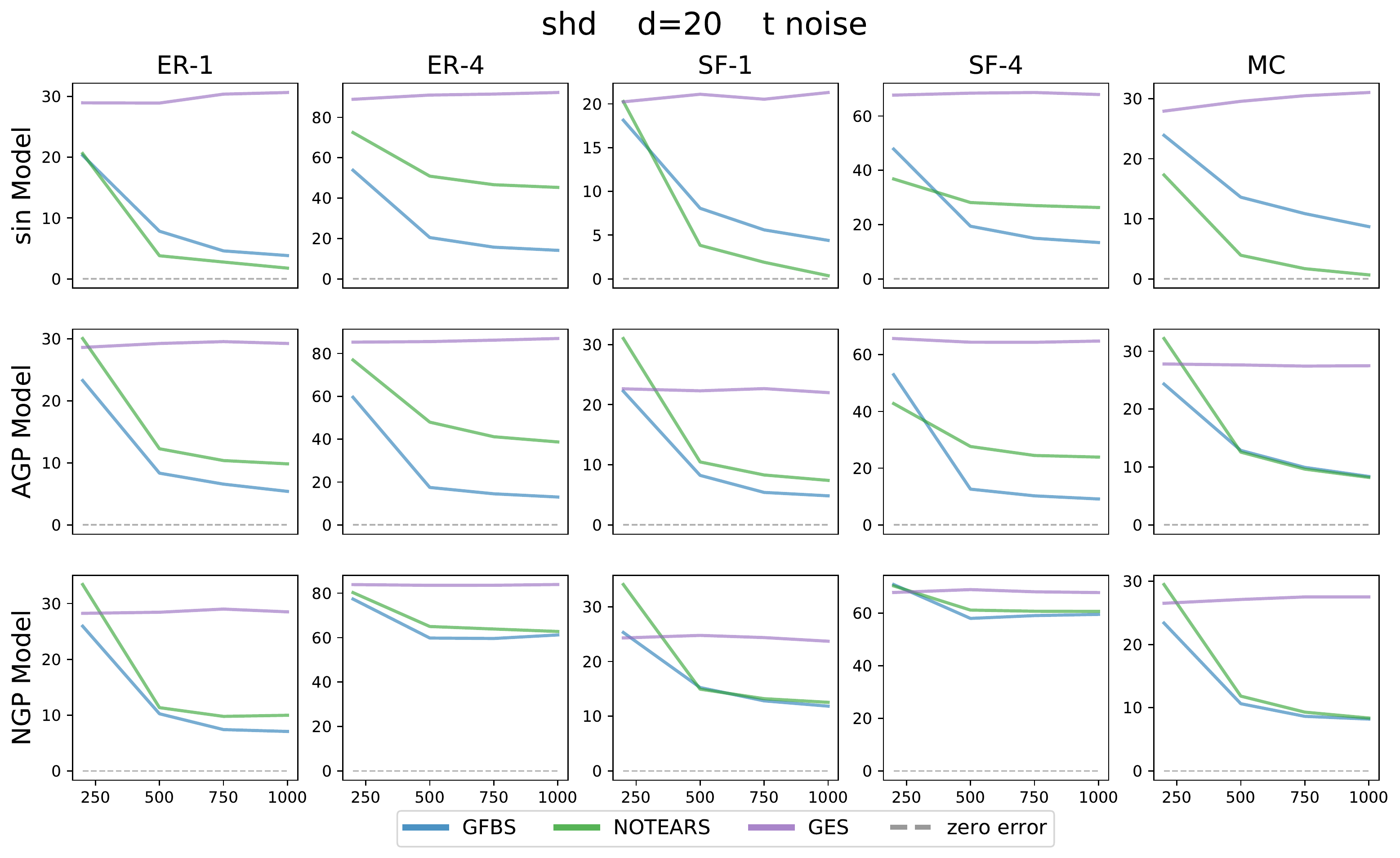}
		\end{subfigure}
		\begin{subfigure}[t]{0.49\textwidth}
			\includegraphics[width=1.\textwidth]{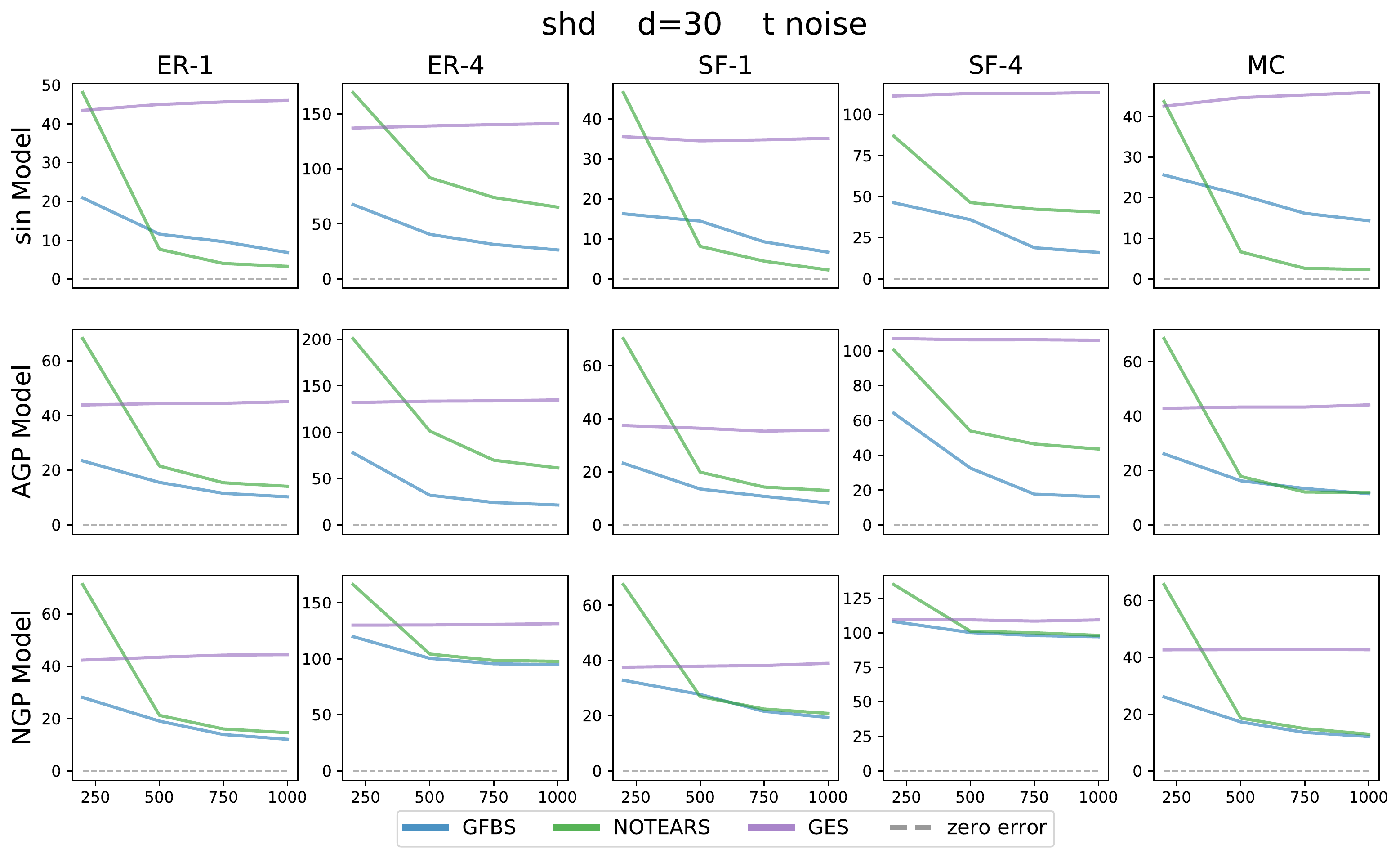}
		\end{subfigure}
		\caption{SHD v.s. sample size for $d=20,30$ and t noise.}
		\label{fig:shd:t}
	\end{figure}

	\begin{figure}[h]
		\centering
		\begin{subfigure}[t]{0.49\textwidth}
			\includegraphics[width=1.\textwidth]{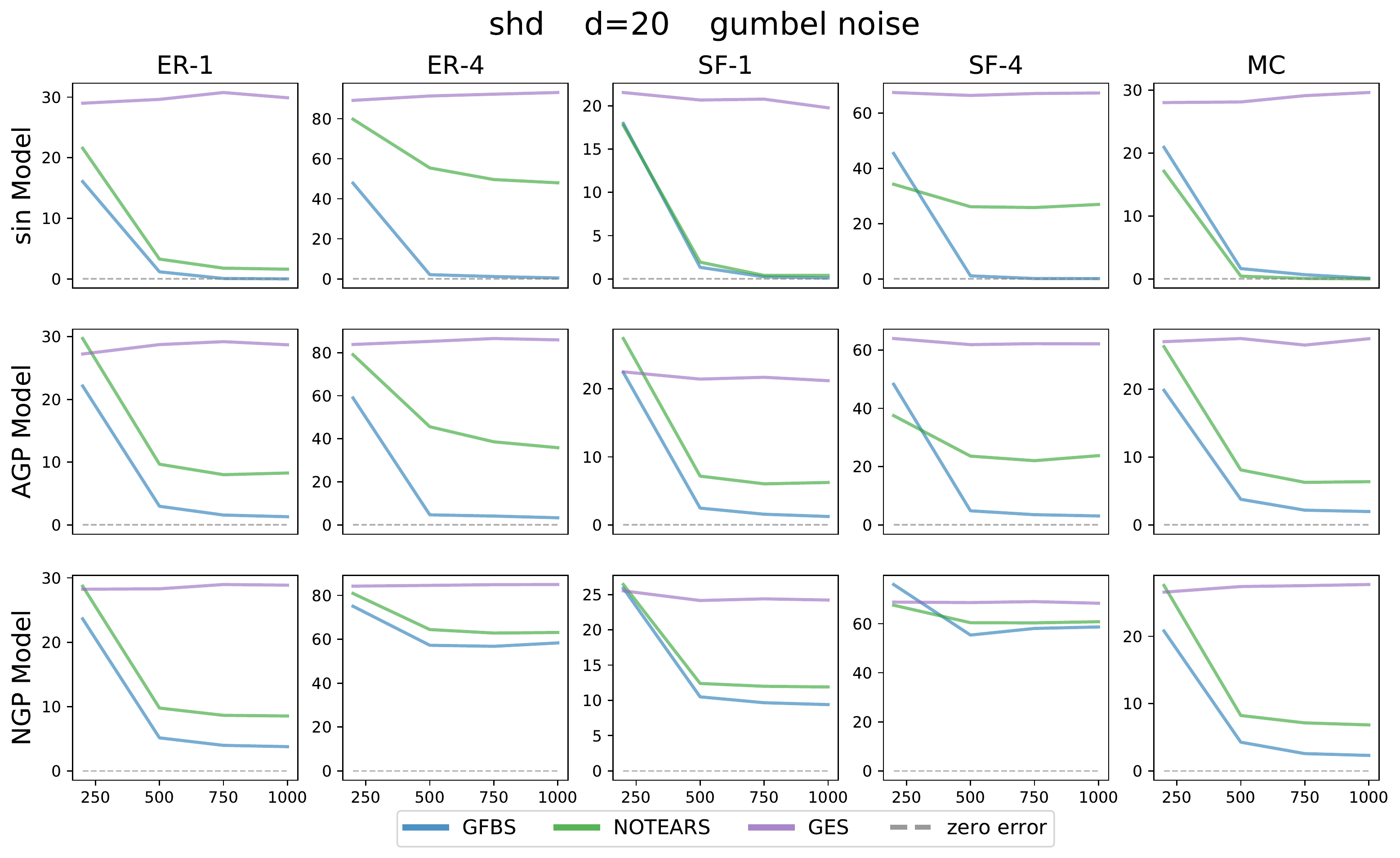}
		\end{subfigure}
		\begin{subfigure}[t]{0.49\textwidth}
			\includegraphics[width=1.\textwidth]{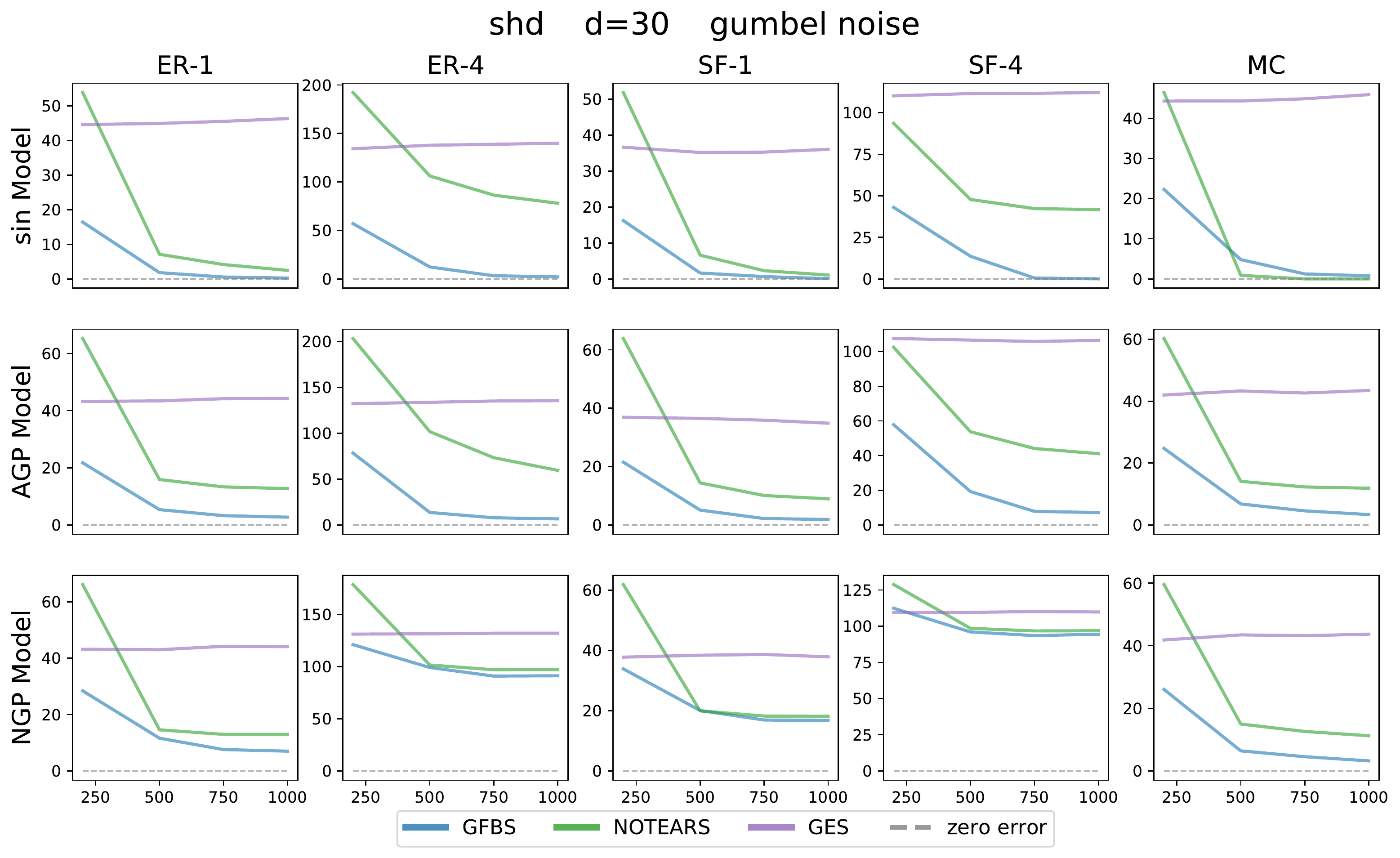}
		\end{subfigure}
		\caption{SHD v.s. sample size for $d=20,30$ and Gumbel noise.}
		\label{fig:shd:gumbel}
	\end{figure}

	\subsubsection{Score optimization}
	We consider the experiment setting under $\phi_1$. Record the estimated score of estimated DAG at each step of iteration ($5$ for $d=5$ in total), starting from empty graph. The different sample size is indicated by darkness of the color. The gray dashed line is the true score ($5$ for $d=5$).
	\begin{itemize}
		\item Figure~\ref{fig:obj_5_gaussian}: Score v.s. iteration for $d=5$ and Gaussian noise
		\item Figure~\ref{fig:obj_5_t}: Score v.s. iteration for $d=5$ and t noise
		\item Figure~\ref{fig:obj_5_gumbel}: Score v.s. iteration for $d=5$ and Gumbel noise
	\end{itemize}

	\begin{figure}[h]
		\centering
		\includegraphics[width=1.\linewidth]{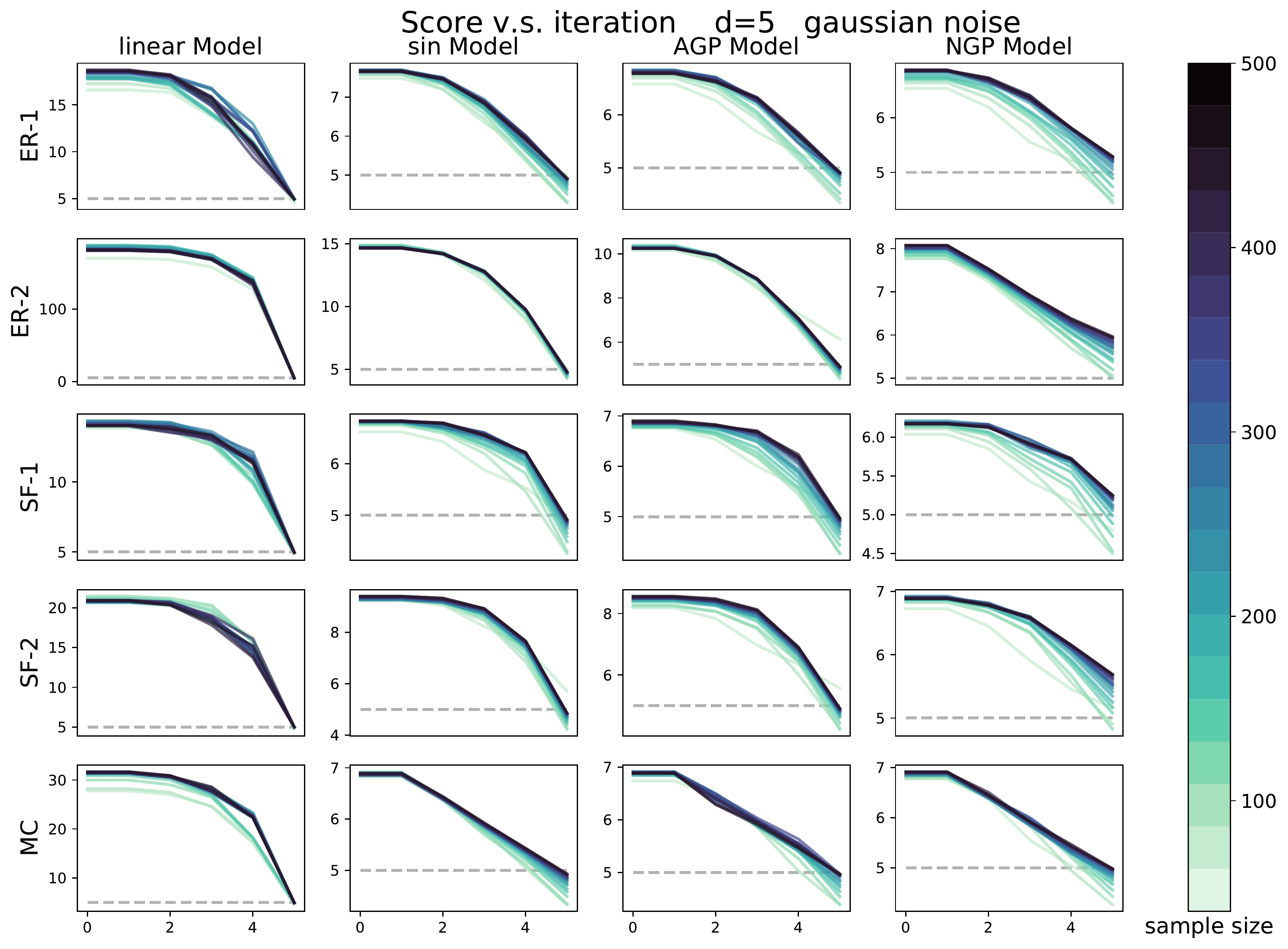}
		\caption{Score v.s. iteration for $d=5$ and Gaussian noise}
		\label{fig:obj_5_gaussian}
	\end{figure}

	\begin{figure}[h]
		\centering
		\includegraphics[width=1.\linewidth]{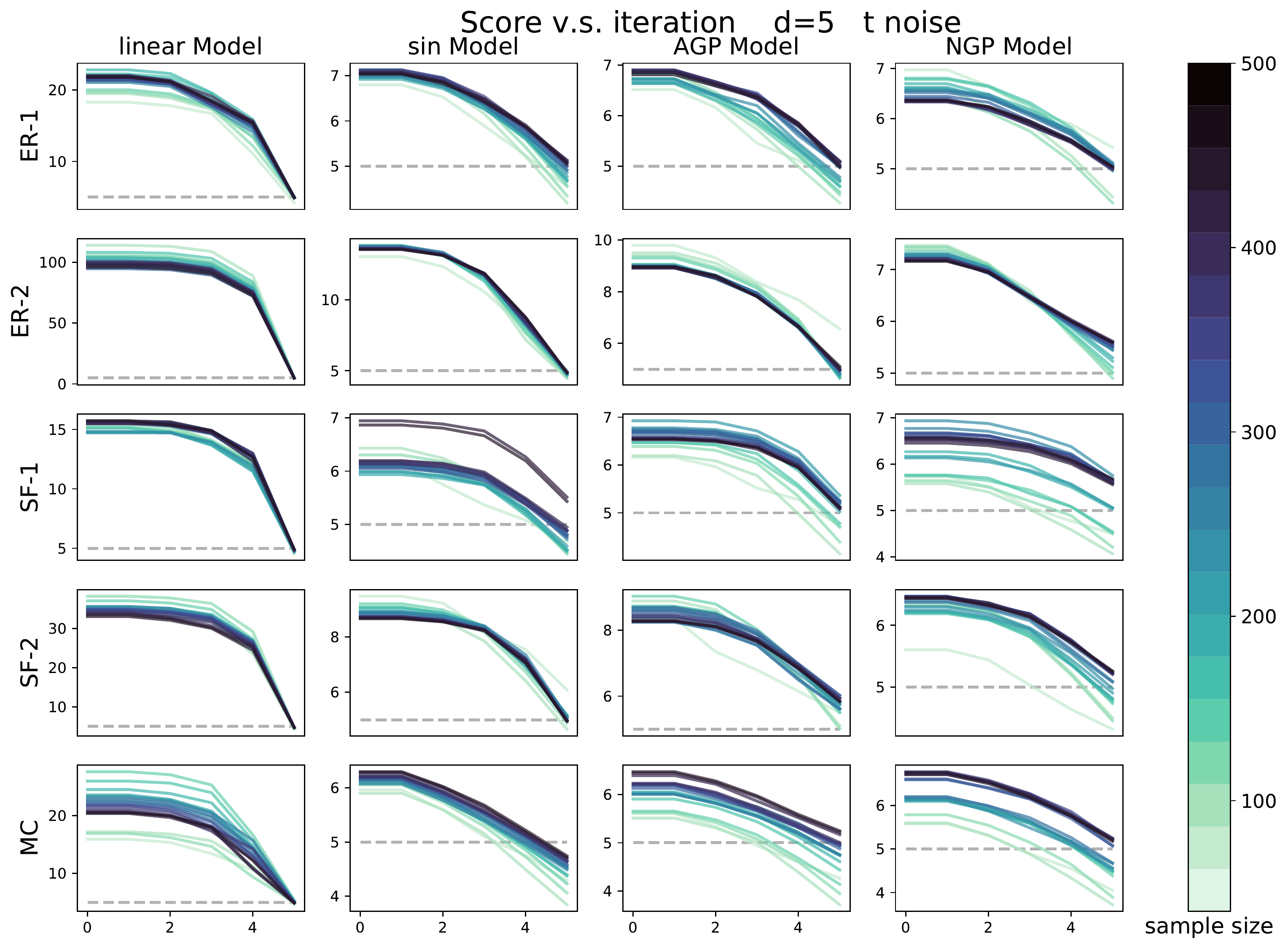}
		\caption{Score v.s. iteration for $d=5$ and t noise}
		\label{fig:obj_5_t}
	\end{figure}

	\begin{figure}[h]
		\centering
		\includegraphics[width=1.\linewidth]{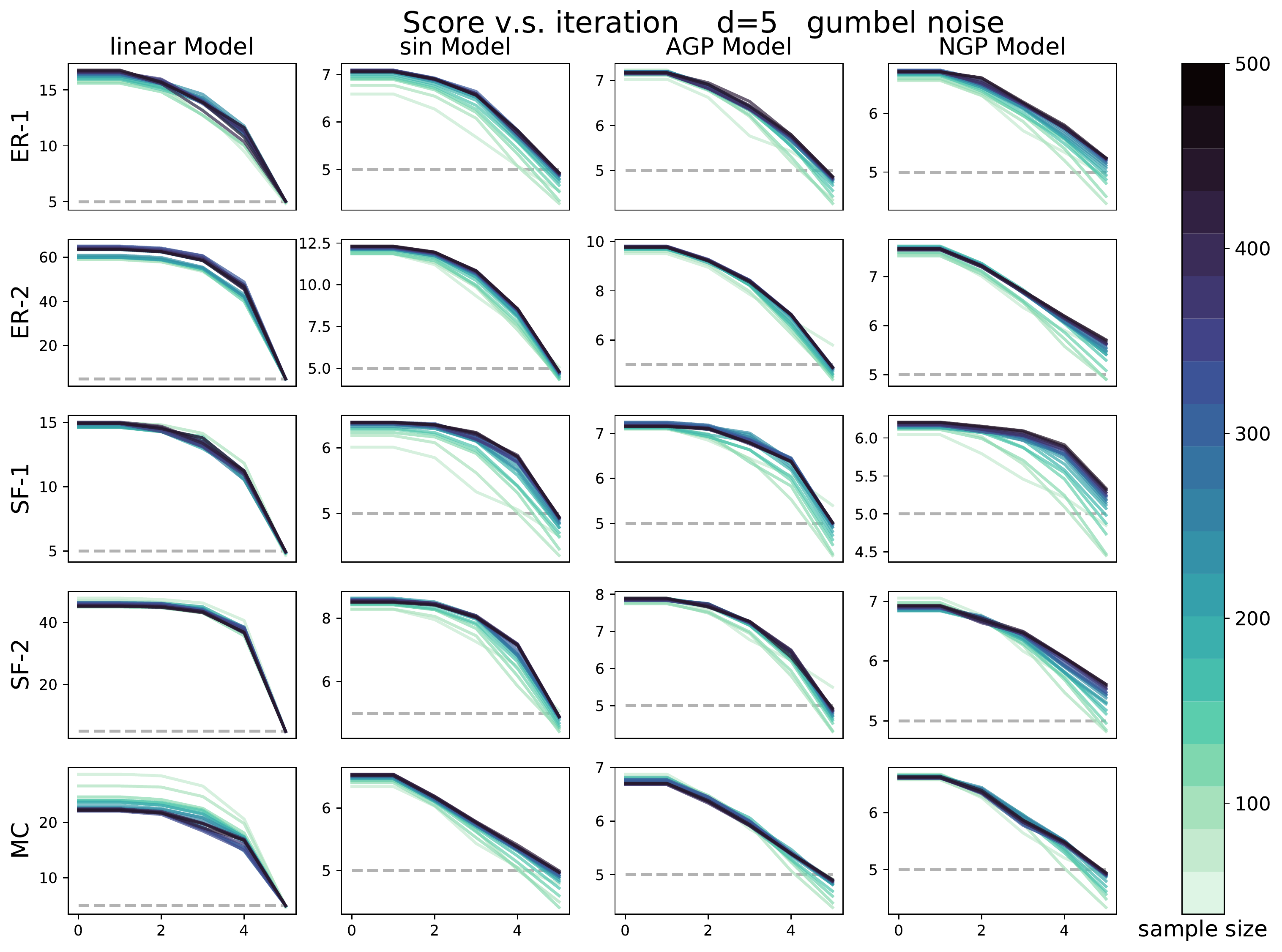}
		\caption{Score v.s. iteration for $d=5$ and Gumbel noise}
		\label{fig:obj_5_gumbel}
	\end{figure}

\end{document}